\documentclass{article}

\usepackage[preprint,nonatbib]{neurips_2019}

\usepackage{collcell}
\usepackage{booktabs}
\usepackage{amssymb, amsmath, amsthm}
\usepackage{color}
\usepackage{floatrow}

\newtheorem{lemma}{Lemma}
\newtheorem{theorem}{Theorem}

\newfloatcommand{capbtabbox}{table}[][\FBwidth]
\usepackage{blindtext}

\usepackage[utf8]{inputenc} 
\usepackage[T1]{fontenc}    
\usepackage{hyperref}       
\usepackage{url}            
\usepackage{booktabs}       
\usepackage{amsfonts}       
\usepackage{nicefrac}       
\usepackage{microtype}      
\usepackage[numbers]{natbib}
\usepackage{tikz}
\usetikzlibrary{shapes.geometric, arrows, calc}

\pgfdeclarelayer{shadow} 
\pgfsetlayers{shadow,main}

\title{Online Markov Decoding: Lower Bounds and Near-Optimal Approximation Algorithms}

\author{%
  Vikas K.~Garg\\
  MIT \\
  \texttt{vgarg@csail.mit.edu} \\
  \And
  Tamar Pichkhadze \\
  MIT \\
  \texttt{tamarp@alum.mit.edu}\\
  }

\usepackage{microtype}
\usepackage{graphicx}
\usepackage{subcaption}
\usepackage{booktabs} 
\usepackage{algorithmic}
\usepackage{tikz}
\usepackage{tkz-berge}
\usetikzlibrary{matrix,chains,positioning,decorations.pathreplacing,arrows}
\usetikzlibrary{bayesnet}
\usepackage{bm}
\usepackage{mathtools}
\usepackage{graphicx}
\usepackage{siunitx}
\usepackage{etoolbox}
\newlength\figH
\newlength\figW
\usepackage[colorinlistoftodos]{todonotes}
\usepackage{pgfplots}  
\usetikzlibrary{external}
\usetikzlibrary{plotmarks}
\usepackage{grffile}
\pgfplotsset{compat=newest} 
\pgfplotsset{plot coordinates/math parser=false} 
\usetikzlibrary{decorations.markings}
\usepackage{url}

\usepackage{amsthm}
 \usepackage[english]{babel}
 \usepackage[utf8]{inputenc}
 \usepackage{amsmath, amssymb}
 \usepackage{mathtools}
 \usepackage{amsfonts}
 \usepackage{graphicx}
 \usepackage{pgfplots}  
 \usepackage{tikz}
 \usetikzlibrary{positioning, arrows}
 \usetikzlibrary{external}
 \usepackage{hyperref}  
 \pgfplotsset{compat=newest} 
 \pgfplotsset{plot coordinates/math parser=false} 
 \usetikzlibrary{plotmarks}
 \usepackage{grffile}
\usepackage[colorinlistoftodos]{todonotes}

\usepackage{natbib}

\newtheorem{thm}{Theorem}

\newcommand{\sunderb}[2]{
  \mathclap{\underbrace{\makebox[#1]{$ $}}_{#2}}
}

\usepackage{microtype}
\usepackage{graphicx}
\usepackage{subcaption}
\usepackage{booktabs}
\usepackage{grffile}
\usepackage{etoolbox}
\makeatletter
\patchcmd\@combinedblfloats{\box\@outputbox}{\unvbox\@outputbox}{}{\errmessage{\noexpand patch failed}}
\makeatother

\usepackage[colorinlistoftodos]{todonotes}

\usepackage[utf8]{inputenc} 
\usepackage[T1]{fontenc}    
\usepackage{url}            
\usepackage{booktabs}       
\usepackage{amsfonts}       
\usepackage{nicefrac}       
\usepackage{microtype}      
\usepackage{amsmath, amsthm, amssymb}
\usepackage{algorithm}
\usepackage{algorithmic}
\usepackage{tikz}
\usetikzlibrary{matrix,chains,positioning,decorations.pathreplacing,arrows}
\usetikzlibrary{bayesnet}
\usepackage{graphicx}
\usepackage{subcaption}

\begin{document}
\maketitle

\begin{abstract}
We resolve the fundamental problem of online decoding with general $n^{th}$ order ergodic Markov chain models. Specifically, we provide deterministic and randomized algorithms whose performance is close to that of the optimal offline algorithm even when latency is small. Our algorithms admit efficient implementation via dynamic programs, and readily extend to (adversarial) non-stationary or time-varying settings. We also establish lower bounds for online methods under latency constraints in both deterministic and randomized settings, and show that no online algorithm can perform significantly better than our algorithms. Empirically, just with latency one, our algorithm outperforms the online step algorithm by over 30\% in terms of decoding agreement with the optimal algorithm on genome sequence data. 
%
\end{abstract}

\section{Introduction} \label{Introduction}
Markov models, in their various incarnations, have for long formed the backbone of diverse applications such as telecommunication \cite{V1967}, biological sequence analysis \cite{G2018}, protein structure prediction \cite{CGW2004}, language modeling \cite{HPCK2013}, automatic speech recognition \cite{B2003}, financial modeling \cite{BB2006}, gesture recognition \cite{WQMD2006}, and  traffic analysis \cite{FHK2003, TRBMG2011}. In a Markov chain model of order $n$, the conditional distribution of next state at any time $i$ depends only on the current state and the previous $n-1$ states, i.e.,  $$ \mathbb{P}(y_{i}|y_1, \ldots, y_{i-1}) ~=~ \mathbb{P}(y_{i}|y_{i-n}, \ldots, y_{i-1})~~ \forall i~.$$
 Often, the states are not directly accessible but need to be inferred or {\em decoded} from the observations, i.e., a sequence of tokens {\em emitted} by the states. For instance, in tagging applications \cite{ATH2003}, each state pertains to a part-of-speech tag (e.g. noun, adjective) and each word $w_i$ in an input sentence $\boldsymbol{w} = (w_1, \ldots, w_T)$ needs to be labeled with a probable tag $y_i$ that might have emitted the word. Thus, it is natural to endow each state with a distribution over the tokens it may emit. For example, $n^{th}$ order {\em hidden Markov models} ($n$-HMM) \cite{R1990} and $(n+1)$-gram language models \cite{HPCK2013} assume the joint distribution $\mathbb{P}({\bf y}, \boldsymbol{w}$) of states ${\bf y} = (y_1, \ldots, y_T)$ and observations $\boldsymbol{w}$ factorizes as 
 \vskip -0.4cm
$$\mathbb{P}({\bf y}, \boldsymbol{w}) ~=~ \prod_{i=1}^T \mathbb{P}(y_i|y_{i-n}, \ldots, y_{i-1}) ~\mathbb{P}(w_i|y_{i})~,$$
\vskip -0.3cm
where $y_{-n+1}, \ldots, y_0$ are dummy states, and the {\em transition} distributions $\mathbb{P}(y_i|y_{i-n}, \ldots, y_{i-1})$ and the {\em emission} distributions $\mathbb{P}(w_i|y_i)$ are estimated from data. We call a Markov model {\em ergodic} if there is a {\em diameter} $\Delta$ such that any state can be reached from any other state in at most $\Delta$ transitions. For instance, a fully connected Markov chain pertains to  $\Delta = 1$.  When the transition distributions do not change with time $i$, the model is called {\em time-homogeneous}, otherwise it is {\em non-stationary}, {\em time-varying} or {\em non-homogeneous} \cite{O2005, PRG2001, CX2011}. Given a sequence $\boldsymbol{w}$ of $T$ observations, the decoding problem is to infer a most probable sequence or {\em path} ${\bf y^*}$ of $T$ states
\begin{eqnarray*} {\bf y^*} ~\in~ \arg\!\max_{\bf y} \mathbb{P}({{\bf y}}, \boldsymbol{w}) ~=~ \arg\!\max_{\bf y} ~\log \mathbb{P}({\bf y}, \boldsymbol{w}) ~.
\end{eqnarray*}
\vskip -0.3cm
Decoding is a key inference problem in other structured prediction settings \cite{TGK2003, THJA2004} as well, e.g., maximum entropy Markov models (MEMM) \cite{MFP2000} and conditional random fields (CRF) \cite{LMP2001, PBJ2009} employ learnable parameters $\boldsymbol{\theta}$ and define the conditional dependence of each state on the observations through feature functions $\boldsymbol{\phi}$. The decoding task in all these models can be expressed in the form 
\vskip -0.45cm
\begin{equation} \label{Viterbi} {\bf y^*} ~\in~ \arg\!\max_{{\bf y}} \sum_{i=1}^{T} \overline{R}_i(y_i|y_{i-n}, \ldots, y_{i-1})~, \end{equation}
\vskip -0.4cm
where we have made the dependence on observations $\boldsymbol{w}$ implicit in the {\em reward} functions $\overline{R}_i$ as shown in Table \ref{reward}. 
\setlength{\tabcolsep}{4pt}
\begin{table}[t]
\caption{Standard Markov models in the reward form. We use  $y_{[i, j]}$ to denote $(y_{i}, y_{i+1}, \ldots, y_j)$. }
\label{reward}
 \begin{center}
 \begin{small}
 \begin{sc}
 \begin{tabular}{|lc|lc|}
 \toprule
\bf Model & \bf Reward $\overline{R}_i(y_i|y_{[i-n, i-1]})$ & \bf Model & \bf Reward $\overline{R}_i(y_i|y_{[i-n, i-1]})$\\
\toprule
$(n+1)$-gram & $\log \mathbb{P}(y_i|y_{i-n}, \ldots, y_{i-1}) + \log  \mathbb{P}(w_i|y_{i})$ &
1-HMM &  $\log \mathbb{P}(y_i|y_{i-1}) + \log \mathbb{P}(w_i|y_i)$\\
\toprule
$n$-MEMM & $\log \dfrac{\exp(\boldsymbol{\theta}^{\top} \boldsymbol{\phi}(y_{[i-n, i-1]}, y_i, \boldsymbol{w}, i))}{\sum_{y_i'} \exp(\boldsymbol{\theta}^{\top} \boldsymbol{\phi}(y_{[i-n, i-1]}, y_i', \boldsymbol{w}, i))} $ &
$n$-CRF & $\boldsymbol{\theta}^{\top} \boldsymbol{\phi}(y_{[i-n, i-1]}, y_i, \boldsymbol{w}, i)$  \\
  \bottomrule
 \end{tabular}
 \end{sc}
 \end{small}
 \end{center}
 \vskip -0.15in
\end{table}
The Viterbi algorithm \cite{V1967} is employed for solving problems of the form \eqref{Viterbi} exactly. However, the algorithm cannot decode any observation until it has processed the entire observation sequence, i.e., computed and stored for each state $s$ a most probable sequence of $T$ states that ends in $s$.
We say an algorithm has {\em latency} $L$ if $L$ is the smallest $B$ such that the algorithm needs to access at most $B+1$ observations $w_i, w_{i+1}, \ldots, w_{i+B}$ to generate the label for observation $w_i$ at any time $i$ during the decoding process. Thus, the latency of Viterbi algorithm on a sequence of length $T$ is $T-1$, which is prohibitive for large $T$, especially in 
memory impoverished systems such as IoT devices \cite{Gupta17, Kumar17, GDX18, ZAFJS2019}. Besides, the algorithm is not suitable for critical scenarios such as patient monitoring, intrusion detection, and credit card fraud monitoring where delay following the onset of a suspicious activity might be detrimental \cite{NVS2006}. 
Moreover, low latency is desirable for tasks such as drug discovery that rely on detecting interleaved coding regions in massive gene sequences.

A lot of effort has thus been, and continues to be, invested into speeding up the Viterbi algorithm, or reducing its memory footprint  \cite{BT2017}. Some prominent recent approaches include fast matrix multiplication \cite{CFR2016}, compression and storage reduction for HMM \cite{SBV2007, CW2008, LMWZ2009}, and heuristics such as beam search and simulated annealing \cite{KFYK2010, DM2005}. Several of these methods are based on the observation that if all the candidate subsequences in the Viterbi algorithm converge at some point, then all subsequent states will share a common subsequence up to that point \cite{BR2008, GDMAOJ2012}.  
However, these methods do not guarantee reduction in latency since, in the worst case, they still need to process all the rewards before producing any output. \cite{NVS2006}  introduced the {\em Online Step Algorithm} (OSA), with provable guarantees, to handle latency in first order models. However, they make a strong assumption that uncertainty in any state label decreases with latency.  This assumption does not hold for important applications such as genome data.  Moreover, OSA does not provide a direct control over the latency, and may require extensive tuning.  \cite{JKKSS2001} introduced algorithms for online server allocation in a first order fully connected non-homogeneous setting. None of these algorithms work for higher order settings that are constrained by latency. We now explain how we provide a complete resolution to the problem of decoding with general ergodic Markov chain models under latency constraints. 
\subsection*{Our contributions}
We introduce efficient, almost optimal deterministic and randomized algorithms for problems of the form \eqref{Viterbi} under latency constraints. Our bounds apply to general settings, e.g., when the rewards vary with time (non-homogeneous settings), or even when they are presented in an adversarial or adaptive manner. Our  guarantees hold for finite latency (i.e. not only asymptotically), and improve with increase in latency. Our algorithms are efficient dynamic programs that may not only be deployed in settings where Viterbi algorithm might be used but also, as we explained earlier, several others where it is impractical. 
Thus, our work would potentially widen the scope of and expedite scientific discovery in several fields that rely critically on efficient online Markov decoding.  

We also provide the first results on the limits of online Markov decoding under latency constraints. Specifically, we craft lower bounds for the online approximation of Viterbi algorithm in both deterministic and randomized ergodic chain settings. Moreover, we establish that no online algorithm can perform significantly better than our algorithms. In particular, our algorithms provide strong guarantees even for low latency, and nearly match the lower bounds for sufficiently large latency.  

We introduce several novel ideas and analyses in the context of approximate Markov decoding. For example, we approximate an undiscounted objective over {\em horizon} $T$ by a sequence of smaller discounted subproblems over horizon $L+1$, and track down the Viterbi algorithm by essentially foregoing rewards on at most $\tilde{\Delta} = \Delta + n - 1$ steps in each smaller problem. We also show how to accomplish effective {\em derandomization} to replicate the optimal expected performance of our randomized algorithm in the deterministic setting. 
  Our design of constructions toward proving lower bounds in a setting predicated on interplay of several heterogeneous variables, namely $L$, $n$, and $\Delta$, is another significant technical contribution.  We believe our tools will foster designing new online algorithms, and establishing combinatorial bounds for related settings such as  dynamic Bayesian networks, hidden semi-Markov models, and model based reinforcement learning. 
\setlength{\tabcolsep}{0.2pt}
\begin{table*}[t]
\caption{Summary of our results in terms of the competitive ratio $\rho$. Note that the effective diameter $\tilde{\Delta} = \Delta + n -1$. To fit some results within the margins, we use the standard notation $\Theta(\cdot)$ on the growth of functions and summarize the performance of Peek Search asymptotically in $L$. The non-asymptotic dependence on $L$ is made precise in all cases in our theorem statements.}
\label{tab:summary}
 \begin{center}
 \begin{small}
 \begin{sc}
 \begin{tabular}{lccc}
 \toprule
 & Lower Bound & Upper Bound (Our Algorithms)  \\
 \toprule
 Deterministic  &  & \\
($\Delta=1, n = 1$)& $1 + \dfrac{1}{L} + \dfrac{1}{L^2 + 1}$ & $\min\left\{\left(1 + \dfrac{1}{L}\right)
\sqrt[\leftroot{-3}\uproot{1} \scriptstyle L ]{L+1}, 1 + \dfrac{4}{L-7}\right\}$\\
\midrule
Randomized & $1 + \dfrac{(1-\epsilon)}{L+\epsilon}$ & $1 + \dfrac{1}{L}$\\
$(\Delta=1, n=1, \epsilon > 0)$ & & \\
\midrule
Deterministic  & $1 + \dfrac{\tilde{\Delta}}{L} \left(1 + \dfrac{\tilde{\Delta} + L -1}{(L - \tilde{\Delta}-1)^2 + 4 \tilde{\Delta}L- 3 \tilde{\Delta}}\right)$ & $1 + \min\bigg\{\Theta\left(\dfrac{\log L}{L - \tilde{\Delta} + 1}\right)~,$ \\
& & $\Theta\left(\dfrac{1}{L - 8 \tilde{\Delta} + 1}\right)\bigg\}$ \\
\midrule
Randomized ($\epsilon > 0$)  & $1 + \dfrac{\left(2^{\Delta-1} \lceil 1/\epsilon \rceil - 1 \right)n}{2^{\Delta-1} \lceil 1/\epsilon \rceil L + n}$  & $1 + \Theta \left(\dfrac{1}{L - \tilde{\Delta} + 1} \right)$\\
 \bottomrule
 \end{tabular}
 \end{sc}
 \end{small}
 \end{center}
 \vskip -0.1in
\end{table*}

\section{Overview of our results}  \label{Overview}
We introduce some notation. We define $[a, b] \triangleq (a, a+1, \ldots, b)$ and $[N] \triangleq (1, 2, \ldots, N)$. Likewise, $y_{[N]} \triangleq (y_1, \ldots, y_N)$ and $y_{[a, b]} \triangleq (y_a, \ldots, y_b)$. We denote the last $n$ states visited by the online algorithm at time $i$ by $\hat{y}_{[i-n, i-1]}$, and those by the optimal offline algorithm by $y^*_{[i-n, i-1]}$.   Defining positive reward functions $R_i = \overline{R}_i + p$ by adding a sufficiently large positive number $p$ to each reward, we note from \eqref{Viterbi} that an optimal sequence of states for input observations $\boldsymbol{w}$ of length $T$ is 
\vspace*{-0.2cm}
\begin{equation} \label{Rewardeq2} y^*_{[1, T]} \in \arg\!\max_{y_1, \ldots, y_T} \sum_{i=1}^{T} R_i(y_i|y_{[i-n, i-1]})~~.
\end{equation}
We use $OPT$ to denote the total reward accumulated by the optimal offline algorithm, and $ON$ to denote that received by the online algorithm. We evaluate the performance of any online algorithm in terms of its {\em competitive ratio} $\rho$, which is defined as the ratio $OPT/ON$. That is,
$$ \rho  ~=~  \sum_{i=1}^T R_i(y_i^*|y^*_{[i-n, i-1]}) \bigg/ \sum_{i=1}^T R_i(\hat{y}_i|\hat{y}_{[i-n, i-1]})~~.$$
Clearly, $\rho \geq 1$. Our goal is to design online algorithms that have competitive ratio close to $1$. For randomized algorithms, we analyze the ratio obtained by taking expectation of the total online reward over its internal randomness.
The performance of any online algorithm depends on the order $n$, latency $L$, and {\em diameter} $\Delta$. Table \ref{tab:summary}  provides a summary of our results. Note that our algorithms are asymptotically optimal in $L$. For the finite $L$ case, we first consider the fully connected first order models. Our randomized algorithm matches the lower bound even\footnote{It is easy to construct examples where any algorithm  with no latency may be made to incur an arbitrarily high $\rho$. Thus, in the fully connected first order Markov setting, online learning is meaningful only for $L \geq 1$.}
with $L=1$ since we may set $\epsilon$ arbitrarily close to 0. Note that just with $L=1$, our deterministic algorithm achieves a competitive ratio 4, and this ratio reduces further as $L$ increases. Moreover our ratio rapidly approaches the lower bound with increase in $L$. Finally, in the general setting, our algorithms are almost optimal when $L$ is sufficiently large compared to  $\tilde{\Delta} = \Delta+n-1$. We call  $\tilde{\Delta}$ the {\em effective diameter} since it nicely encapsulates the roles of order $n$ and diameter $\Delta$ toward the quality of approximation. 

 The rest of the paper is organized as follows. We first introduce and analyze our deterministic {\em Peek Search} algorithm for homogeneous settings in section \ref{PeekSearchDesc}.  We then introduce the {\em Randomized Peek Search} algorithm in section \ref{RandomPeekDesc}. We propose the deterministic {\em Peek Reset} algorithm, in section \ref{PeekResetDesc}, that performs better than deterministic Peek Search for large $L$.  
 We then present the  lower bounds in section \ref{LowerBoundsDesc}. Finally, we 
demonstrate the merits of our approach via
strong empirical evidence on the problem of genome sequencing in section \ref{Experiments}.  For improved readability, we provide all the proofs in the supplementary material.  We also show how our guarantees extend to the non-homogeneous settings, and outline an efficient dynamic program in the supplementary material. 
\section{Peek Search} \label{PeekSearchDesc}
Our idea is to approximate the sum of rewards over $T$ steps in \eqref{Rewardeq2} by a sequence of smaller problems over $L+1$ steps. The Peek Search algorithm is so named since at each time $i$, besides the observation $w_i$, it {\em peeks} into the next $L$ observations 
$w_{i+1}, \ldots, w_{i+L}$. The algorithm then leverages the sub-sequence $w_{[i, i+L]}$ to
decide its next state $\hat{y}_i$. Let $\gamma \in (0, 1)$ be a {\em discount factor}. Specifically, the algorithm repeats the following procedure at each time $i$. 
First, it finds a path of length $L+1$ emanating from the current state $\hat{y}_{i-1}$ that fetches maximum {\em discounted} reward. The discounted reward on any path is computed by scaling down the $\ell^{th}$ edge, $\ell \in \{0, \ldots, L\}$, on the path  by $\gamma^{\ell}$. Then, the algorithm moves to the first state of this path and repeats the procedure at time $i+1$. Note that at time $i+1$, the algorithm need not continue with the second edge on the optimal discounted path computed at previous time step $i$, and is free to choose an alternative path. Formally, at time $i$, the algorithm computes $\tilde{y}_i \triangleq (\tilde{y}_i^0, \tilde{y}_i^1, \ldots, \tilde{y}_i^L)$ that maximizes the following objective over valid paths $y = (y_i, \ldots, y_{i+L})$,
 \begin{eqnarray*}   R(y_i|\hat{y}_{[i-n, i-1]}) ~+~  \sum_{j=1}^{n-1} \gamma^j R(y_{i+j}|\hat{y}_{[i-n+j, i-1]}, y_{[i, i+j-1]})  ~+~ \sum_{j=n}^L \gamma^j R(y_{i+j}|y_{[i+j-n, i+j-1]})~, \end{eqnarray*}
sets the next state $\hat{y}_i =  \tilde{y}_i^0$, and receives the reward $R(\hat{y}_i|\hat{y}_{[i-n, i-1]})$. Note that we have dropped the subscript $i$ from $R_i$ since in the homogeneous settings, the reward functions do not change with time $i$. For any given $L$ and $\tilde{\Delta}$, we optimize to get the optimal $\gamma$.  Intuitively, $\gamma$ may be viewed as an {\em explore-exploit} parameter that indicates the confidence of the online algorithm in the best discounted path: $\gamma$ grows as $L$ increases, and thus a high value of $\gamma$ indicates that the path computed at a time $i$ may be worth tracing at subsequent few steps as well. In contrast, the algorithm is uncertain for small values of $L$. We have the following near-optimal result on the performance of Peek Search. 
\begin{figure*}[t]
\centering
    \input{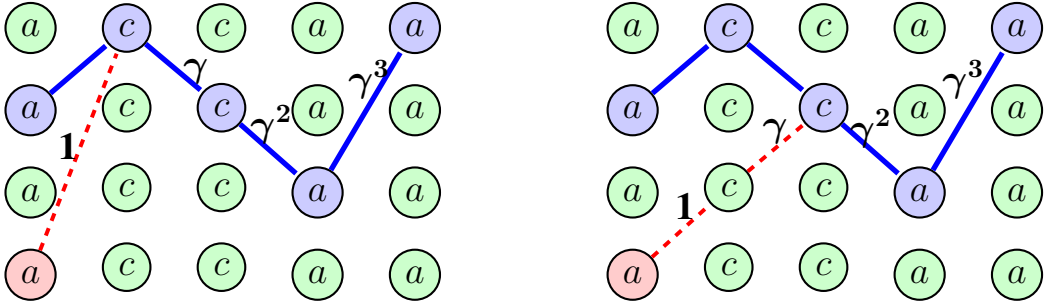}
    \caption{{\bf Visual intuition for the setting $n=1$}. ({\bf Left}) A trellis diagram obtained by unrolling a fully connected Markov graph (i.e. diameter $\Delta = 1$). The states are shown along the rows, and time along the columns.  The system is currently in state 4 (shown in red), and has access to rewards and observations (shown inside circles) for the next $(L+1)$ steps. The unknown optimal path is shown in blue, and the weights with which rewards are scaled are shown on the edges. One option available with the online algorithm is to jump to state 1 (possibly fetching zero reward) and then follow the optimal path for the subsequent $L$ steps. Note that the online algorithm might choose a different path, but it is guaranteed at least as much reward since it maximizes the discounted reward over $L+1$ steps. $\gamma$ approaches 1 with increase in $L$. This would ensure that the online algorithm makes nearly the most of $L$ steps every $L+1$ steps.   ({\bf Right}) If the graph is not fully connected, some of the transitions may not be available (e.g. state 4 to state 1 in our case). Therefore, the online algorithm might not be able to join the optimal path in one step, and thus may have to forgo additional rewards. 
    }
    \label{fig:intuition}
\end{figure*}
\begin{thm} \label{TheoremThird}
The competitive ratio of Peek Search on Markov chain models of order $n \geq 1$ with diameter $\Delta \geq 1$ for  $L \geq \Delta + n - 1$ is
\begin{eqnarray*} \rho & ~\leq~  \dfrac{L+1}{L-\Delta-n+2} \left(\dfrac{L+1}{\Delta+n-1}\right)^{(n + \Delta - 1)/(L - \Delta - n + 2)}
& ~=~ 1 + \Theta \left(\dfrac{\log L}{L - \tilde{\Delta} + 1} \right)~.
\end{eqnarray*}
\end{thm}
\begin{proof} (Sketch)
We first consider the fully connected first order setting (i.e. $n = 1, \Delta = 1$). Our analysis hinges on two important facts. Since Peek Search chooses a path that maximizes the total discounted reward over next $(L+1)$ steps, it is guaranteed to fetch all of the discounted reward pertaining to the optimal path except that available on the first step of the optimal path (see Fig. \ref{fig:intuition} for visual intuition). Alternatively, Peek Search could have persisted with the maximizing path computed at the previous time step (recall that only first step of this path was taken to reach the current state). We exploit the fact that this path is now worth $1/\gamma$ times its anticipated value at the previous step.

Now consider $n > 1$. The online algorithm may jump to any state on the optimal offline, i.e. Viterbi path, in one step. However, the reward now depends on the previous $n$ states, and so the online algorithm may have to wait additional $n-1$ steps before it could trace the subsequent optimal path. Finally, as explained in Fig. \ref{fig:intuition}, when $\Delta > 1$, the online algorithm may have to forfeit rewards on (at most) $\Delta$ steps, in addition to the $n-1$ steps, in order to join the optimal path. 
\end{proof}
We show in the supplementary material that this guarantee on the performance of Peek Search extends to the non-homogeneous settings, including those where the rewards may be adversarially chosen. Note that na\"ively computing a best path by enumerating all paths of of length $L+1$ would be computationally prohibitive since the number of such paths is exponential in $L$. Fortunately, we can design an efficient dynamic program for Peek Search. Specifically, we can show that for every $\ell \in \{1, 2, \ldots, L\}$, the reward on the optimal discounted path of  length $\ell$ can be recursively computed from  an optimal path of length $\ell$-$1$ using $O(|K|^n)$ computations. We have the following result. 
\begin{thm} \label{TheoremNinth}
Peek Search can compute a best $\gamma$-discounted path for the next $L+1$ steps, in $n^{th}$ order Markov chain models, in time $O(L|K|^{n})$, where $K$ is the set of states.
\end{thm}
We outline an efficient procedure,  underlying Theorem \ref{TheoremNinth}, in the supplementary material. 
We now introduce two algorithms that do not recompute the paths at each time step. These algorithms provide even tighter (expected) approximation guarantees than Peek Search for larger values of the latency $L$. 
\section{Randomized Peek Search} \label{RandomPeekDesc}
We first introduce the {\em Randomized Peek Search} algorithm, which removes the asymptotic log factor from the competitive ratio in Theorem \ref{TheoremThird}. Unlike Peek Search, this method does not discount the rewards on paths. Specifically,  the algorithm first selects a {\em reset point} $\ell$ uniformly at random from $\{1, 2, \ldots, L+1\}$. This number is a private information for the online algorithm. The randomized algorithm recomputes the optimal {\em non-discounted} path (which corresponds to $\gamma = 1$) of length $(L+1)$, once every $L+1$ steps, at each time $i*(L+1)+\ell$, and follows this path for next $L+1$ steps without any updates. We have the following result that underscores the benefits of randomization.   
\begin{thm} \label{TheoremSeventh}
Randomized Peek Search achieves, in expectation, on Markov chain models of order $n$ with diameter $\Delta$ a competitive ratio
\begin{eqnarray*}\rho   & ~~\leq~~    1 + \dfrac{\Delta + n - 1}{L+1 - (\Delta + n - 1)}
& ~~=~~   1 + \Theta \left(\dfrac{1}{L - \tilde{\Delta} + 1} \right)~.
\end{eqnarray*}
\end{thm}
\begin{proof}(Sketch)
Since it maximizes the non-discounted reward, for each random reset point $\ell$, the online algorithm receives at least as much reward as the optimal offline algorithm minus the reward on at most $\tilde{\Delta}$ steps every $L+1$ steps. We show that, in expectation, Peek Reset misses on only (at most) a $\tilde{\Delta}/(L+1)$ fraction of the optimal offline reward.
\end{proof}

Theorem \ref{Theorem2} is essentially tight since it  nearly matches the lower bound as described previously in section \ref{Overview}.
We leverage insights from Randomized Peek Search to translate its almost optimal expected performance to the deterministic setting. Specifically, we introduce the Peek Reset algorithm that may be loosely viewed as a {\em derandomization} of Randomized Peek Search. The main trick is to conjure a sequence of reset points, each over a variable number of steps.  This allows the algorithm to make adaptive decisions about when to forgo rewards.  Both Randomized Peek Search and Peek Reset can compute rewards on their  paths efficiently by using the procedure for Peek Search as a subroutine.
\section{Peek Reset} \label{PeekResetDesc}
We now present the deterministic {\em Peek Reset} algorithm that performs better than Peek Search when the latency $L$ is sufficiently large. Like Randomized Peek Search, Peek Reset recomputes a best non-discounted path and takes multiple steps on this path. However, the number of steps taken is not fixed to $L+1$ but may vary in each {\em phase}.  Specifically, let $(i)$ denote the time at which phase $i$ begins.  The algorithm follows, in phase $i$, a sequence of states $\hat{y}(i) \triangleq (\hat{y}_{(i)}, \hat{y}_{(i)+1}, \ldots, \hat{y}_{T_i-1})$ that maximizes the following objective over valid paths $y = (y_{(i)}, \ldots, y_{T_i-1})$ :
\begin{eqnarray*}
 f(y) ~~\triangleq~~  R(y_{(i)}|\hat{y}_{[(i)-n, (i)-1]}) 
 & + &  \sum_{j=1}^{n-1} R(y_{(i)+j}|\hat{y}_{[(i)-n+j, (i)-1]}, y_{[(i), (i)+j-1]}) \\
& + &  \sum_{j=n}^{T_i-(i)-1} R(y_{(i)+j}|y_{[(i)+j-n, (i)+j-1]})~,
\end{eqnarray*}
where $T_i$ is chosen from the following set (breaking ties arbitrarily)
$$\arg\!\!\!\!\!\!\!\!\!\!\!\!\!\min_{t \in [(i) + L/2 + 1], (i)+L]} \max_{(y_{t-n}, \ldots, y_t)} R(y_t|y_{[t-n, t-1]}) ~.$$
Then, the next phase $(i+1)$ begins at time $T_i$. We have the following result. 

\begin{thm} \label{TheoremEighth}
The competitive ratio of Peek Reset on Markov chain models of order $n$ with diameter $\Delta$~ for latency $L$ is 
\begin{eqnarray*} \rho & ~~\leq~~  1 ~+~ \dfrac{2(\Delta+n)(\Delta+n-1)}{L-8(\Delta+n-1)+1} & ~~=~~  1 + \Theta \left(\dfrac{1}{L - 8\tilde{\Delta} + 1} \right)~.
\end{eqnarray*}
\end{thm}
\begin{proof}
(Sketch) The algorithm gives up reward on at most $\tilde{\Delta}$ steps every $L+1$ steps, however these steps are cleverly selected. Note that $T_i$ is chosen from the interval $[(i) + L/2 + 1], (i)+L]$, which contains steps from both phases $(i)$ and $(i+1)$. Thus, the algorithm gets to peek into phase $(i+1)$ before deciding on the number of steps to be taken in phase $(i)$.   
\end{proof}
\begin{figure*}[t]
\hskip -0.1cm
\begin{subfigure}{0.463\textwidth}
    \begin{tikzpicture}[rotate=-30,scale=.67]
\GraphInit[vstyle=Shade]
\SetVertexNoLabel
\renewcommand*{\VertexBallColor}{red}
\begin{scope}
\grEmptyCycle[RA=2.7,  prefix=a]{1}
\begin{scope}[rotate=120]
\grEmptyCycle[RA=2.7,  prefix=i]{1}
\end{scope}
\renewcommand*{\VertexBallColor}{magenta}
\begin{scope}[rotate=240]
\grEmptyCycle[RA=3,  prefix=j]{1}
\end{scope}
\Edge(j0)(i0)
\Edge(j0)(a0)
\Edge(i0)(a0)
\renewcommand*{\VertexBallColor}{orange}
\renewcommand*{\EdgeColor}{blue}
\grEmptyCycle[RA=1,  prefix=b]{1}
\begin{scope}[rotate=120]
\grEmptyCycle[RA=1,  prefix=k]{1}
\end{scope}
\renewcommand*{\VertexBallColor}{red}
\begin{scope}[rotate=240]
\grEmptyCycle[RA=1,  prefix=l]{1}
\end{scope}
\Edge(b0)(k0)
\Edge(b0)(l0)
\Edge(l0)(k0)
\renewcommand*{\EdgeColor}{red}
\Edge(a0)(b0)
\Edge(i0)(k0)
\Edge(j0)(l0)
\end{scope}

\renewcommand*{\VertexBallColor}{magenta}
\begin{scope}[xshift=3 cm, yshift=5cm]
\grEmptyCycle[RA=3,  prefix=c]{1}
\begin{scope}[rotate=120]
\grEmptyCycle[RA=2.7,  prefix=g]{1}
\end{scope}
\renewcommand*{\VertexBallColor}{green!50!black}
\begin{scope}[rotate=240]
\grEmptyCycle[RA=3,  prefix=h]{1}
\end{scope}

\renewcommand*{\EdgeColor}{blue}
\begin{scope}[rotate=-30, yshift=0.5cm]
\renewcommand*{\VertexBallColor}{red}
\grComplete[RA=0.8,  prefix=d]{2}
\end{scope}
\begin{scope}[yshift=-0.7cm, xshift=-0.5cm]
\renewcommand*{\VertexBallColor}{magenta}
\grComplete[RA=0.1,  prefix=e]{1}
\end{scope}
\renewcommand*{\EdgeColor}{red}
\EdgeIdentity*{c}{d}{0}
\Edge(h0)(e0)
\Edge(c0)(h0)
\Edge(c0)(g0)
\Edge(d1)(g0)
\renewcommand*{\EdgeColor}{black}
\Edge(h0)(g0)
\Edge(c0)(g0)
\Edge(c0)(h0)
\renewcommand*{\EdgeColor}{blue}
\Edge(d0)(e0)
\Edge(d1)(e0)
\node[label={\Large $A$}] (A) at (-1.5,-4) {};
\node[label={\Large $B$}] (B) at (3.1,0.57) {};
\node[label={\Large $C$}] (C) at (-0.2,2.3) {};
\node[label={\Large $C'$}] (C') at (-4.9,-3.5) {};
\node[label={\Large $A'$}] (A') at (-4.7,-7.4) {};
\node[label={\Large $B'$}] (B') at (0.3,-4.5) {};

\node[label={\Large $a$}] (a) at (-0.8,-1.6) {};
\node[label={\Large $b$}] (b) at (1.85,0) {};
\node[label={\Large $c$}] (c) at (-0.3,1.2) {};
\node[label={\Large $a'$}] (a') at (-3.7,-5.7) {};
\node[label={\Large $b'$}] (b') at (-2.2,-4.8) {};
\node[label={\Large $c'$}] (c') at (-3.8,-4.1) {};


\end{scope}
\renewcommand*{\EdgeColor}{cyan}
\EdgeIdentity*{a}{c}{0}
\Edge(i0)(g0)
\Edge(d1)(k0)
\end{tikzpicture}








    \caption{{\bf Deterministic setting}}
    \label{fig:my_label2}
\end{subfigure}
 \hspace*{-0.25cm}
\begin{subfigure}{0.52\textwidth}
    \begin{tikzpicture}[rotate=-30,scale=.65]
\GraphInit[vstyle=Shade]
\SetVertexNoLabel
\begin{scope}
\grEmptyCycle[RA=3,  prefix=a]{1}
\begin{scope}[rotate=120]
\grEmptyCycle[RA=2.7,  prefix=i]{1}
\end{scope}
\begin{scope}[rotate=240]
\renewcommand*{\VertexBallColor}{green!50!black}
\grEmptyCycle[RA=3,  prefix=j]{1}
\end{scope}
\GraphInit[vstyle=Shade]
\Edge(j0)(i0)
\Edge(j0)(a0)
\Edge(i0)(a0)
\renewcommand*{\EdgeColor}{blue}
\grComplete[RA=1,  prefix=b]{3}
\renewcommand*{\EdgeColor}{red}
\EdgeIdentity*{a}{b}{0}
\Edge(i0)(b1)
\Edge(j0)(b2)
\end{scope}

\begin{scope}[xshift=2.6 cm, yshift=5cm]
\grComplete[RA=3,  prefix=c]{3}
\renewcommand*{\EdgeColor}{blue}
\grComplete[RA=1,  prefix=d]{3}
\renewcommand*{\EdgeColor}{red}
\EdgeIdentity*{c}{d}{0,...,2}
\end{scope}
\renewcommand*{\EdgeColor}{cyan}
\Edge(i0)(c1)
\EdgeIdentity*{a}{c}{0}
\EdgeIdentity*{b}{d}{1}

\end{tikzpicture}








    \caption{\bf{Randomized setting}}
    \label{fig:my_label3}
\end{subfigure}
\caption{{\bf Constructions for lower bounds with $\Delta=3$}. ({\bf Left}) $ABC$ and $abc$ are opposite faces of a triangular prism, and $A'B'C'$ and $a'b'c'$ are their translations. The resulting prismatic polytope  has the property that distance between the farthest vertices is $\Delta$.   
Different colors are used for edges on different faces, and same color for translated faces to aid visualization (we have also omitted some edges that connect faces to their translated faces, in order to avoid clutter). A priori the rewards for the $L+1$ steps are same across all vertices (i.e. states). Thus, due to symmetry of the polytope, the online algorithm arbitrarily chooses some vertex (shown here in green). The states that can be reached via shortest paths of same length from this vertex are displayed in same color (magenta, red, or orange). 
The adversary reveals the rewards for an additional, i.e. $(L+2)^{th}$, time step such that states at distance $d \in [\Delta]$ from the green state would fetch $(n+d-1)\alpha$ for some $\alpha$, while the green state would yield 0. Under the Markov dependency rule that a state yields reward only if it has been visited $n$ consecutive times, the online algorithm fails to obtain any reward in the $(L+2)^{th}$ step regardless of the state sequence it traces. The optimal algorithm, due to prescience, gets the maximum possible reward $(n+\Delta-1)\alpha$ for this step. ({\bf Right}) In the randomized setting, all states fetch zero reward at the final step except a randomly chosen state (shown in green) that yields reward $n$. The probability that the randomized online algorithm correctly guesses the green state at the initial time step is exponentially small in $\Delta$. In all other cases, it must forgo this reward, and thus its expected reward is low compared to the optimal algorithm for large $\Delta$.}
\label{fig:manmade}
\end{figure*}
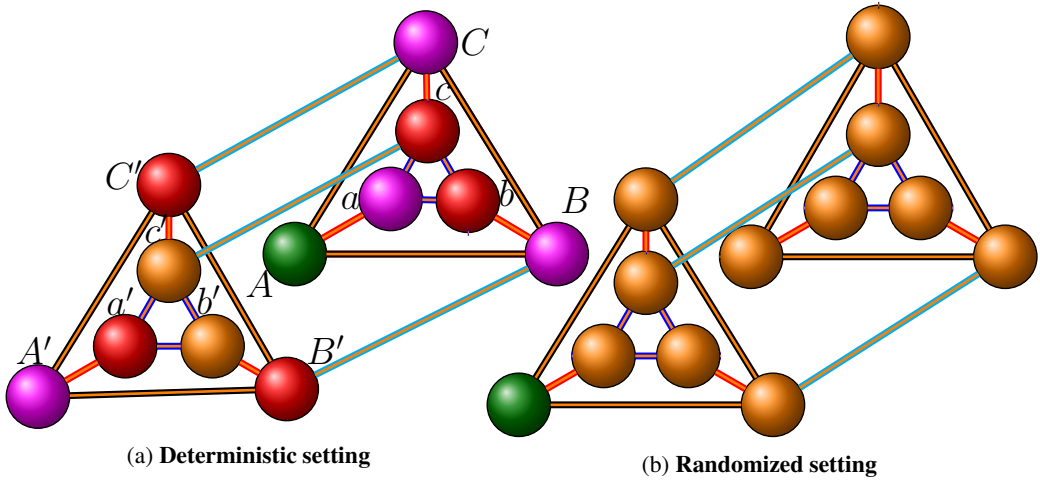

A comparison of Theorem \ref{TheoremEighth} with Theorem \ref{TheoremThird} reveals that Peek Reset provides better upper bounds on the approximation quality
than Peek Search for sufficiently large latency. In particular, for the fully connected first order setting, i.e. $\tilde{\Delta} = 1$, the competitive ratio of Peek Reset is at most $1 + 4/(L-7)$ which is better than the corresponding worst case bound for Peek Search when $L \geq 50$.  
Thus, Peek Search is better suited for applications with severe latency constraints whereas Peek Reset may be preferred in less critical scenarios. 
We now establish that no algorithm, whether deterministic or randomized, can provide significantly better guarantees than our algorithms under latency constraints.
\section{Lower Bounds} \label{LowerBoundsDesc}
We state our lower bounds on the performance of any deterministic and any randomized algorithm in the general non-homogeneous ergodic Markov chain models.  
\begin{thm}\label{TheoremFifth}
The competitive ratio of any deterministic online algorithm on $n^{th}$ order (time-varying) Markov chain models with diameter $\Delta$ for latency $L$ is greater than
$$1 + \dfrac{\tilde{\Delta}}{L} \left(1 + \dfrac{\tilde{\Delta} + L -1}{(\tilde{\Delta}+L-1)^2 + \tilde{\Delta}}\right)~~.$$
In particular, when $n=1$, $\Delta=1$, the ratio is larger than 
$1 + \dfrac{1}{L} + \dfrac{1}{L^2 + 1}$~.
\end{thm}
\begin{proof}(Sketch) The proof revolves around an intricate $\Delta$-dimensional prismatic polytope construction, where each vertex corresponds to a state. In particular, the proof hinges on disentangling the interplay between $L$, $\Delta$, and $n$. The visual intuition is sketched in Fig. \ref{fig:manmade}.  
\end{proof}
\begin{thm} \label{TheoremSixth}
For any $\epsilon > 0$, the competitive ratio of any randomized online algorithm, that is allowed latency $L$, on $n^{th}$ order (time-varying) Markov chain models with $\Delta = 1$ is at least $1 + \dfrac{(1-\epsilon)n}{L+\epsilon n}~.$
For a general diameter $\Delta$, the competitive ratio is at least 
$ 1 + \dfrac{\left(2^{\Delta-1} \lceil 1/\epsilon \rceil - 1 \right)n}{2^{\Delta-1} \lceil 1/\epsilon \rceil L + n}~.$
\end{thm}
\begin{proof}(Sketch)
The proof relies on another prismatic polytope construction. We make it hard to guess a randomly chosen state that fetches high reward. Fig. \ref{fig:manmade} elucidates the setting $\Delta=3$. 
\end{proof}
We now analyze the performance of our algorithms in the wake of these lower bounds. Note that when $\tilde{\Delta} = 1$, Randomized Peek Search (Theorem \ref{TheoremSeventh})
 matches the lower bound  in Theorem \ref{TheoremSixth} even with $L=1$, since we may set $\epsilon$ arbitrarily close to 0.
 Similarly, in the deterministic setting, Peek Search  
 achieves a competitive ratio of 4 with $L=1$ (Theorem \ref{Theorem3}), which is within twice the theoretically best possible performance (i.e. a ratio of 2.5) as specified by Theorem \ref{Theorem5}.
 Moreover, the performance approaches  the lower bound with increase in $L$ as Peek Reset takes center stage (Theorem  \ref{TheoremEighth}).  In the general setting, our algorithms are almost optimal when $L$ is sufficiently large compared to  $\tilde{\Delta}$. We proceed to our empirical findings that accentuate the practical implications of our work.   
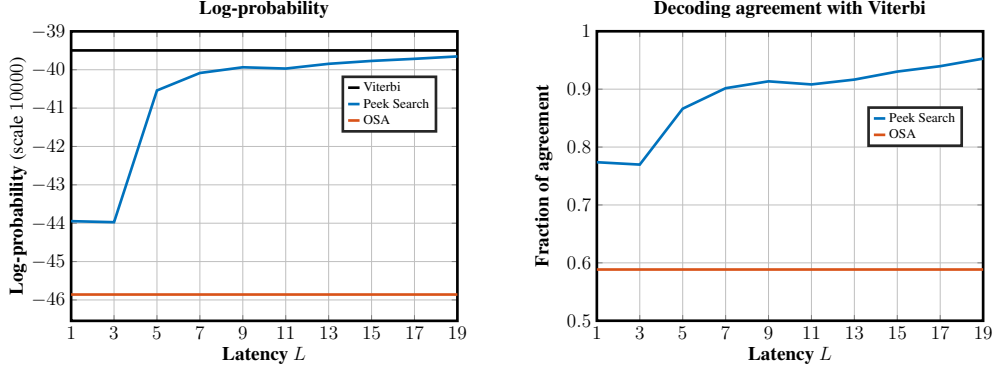
\begin{figure*}[t]
    \begin{subfigure}{0.45\linewidth}
    \centering
    \resizebox{\textwidth}{!}{\definecolor{mycolor1}{rgb}{0.00000,0.44700,0.74100}%
 \definecolor{mycolor2}{rgb}{0.85000,0.32500,0.09800} 
 \begin{tikzpicture} 
 \begin{axis}[%
 width= \figW, 
 height= \figH, 
 at={(1.011in,0.642in)}, 
 scale only axis, 
 xmin=1,  
 xmax=19, 
 ymax=-39,
 xlabel={{\bf Latency $L$}},  
 ytick = {-47, -46, -45, -44, -43, -42, -41, -40, -39},
 xtick = {1, 3, 5, 7, 9, 11, 13, 15, 17, 19},
 yticklabel style={/pgf/number format/precision=3}, 
 line width=2pt,  
 grid=both,   
 grid style={line width=.3pt, draw=gray!10},  
 major grid style={line width=.2pt,draw=gray!50}, 
 ylabel={{\bf Log-probability} (scale $10000$)},   
 axis background/.style={fill=white}, 
 title style= {font=\Large}, 
 title={{\bf Log-probability}}, 
 legend style={at={(0.7, 0.65)},
 anchor= south west,legend cell align=left,align=left,draw=white!15!black}, 
 xlabel style={font=\Large},ylabel style={font=\Large}, ticklabel style={font=\Large},y label style={at={(axis description cs:-0.1,0.5)}},x label style={at={(axis description cs:0.5, -0.07)}},legend style={legend cell align=left,align=left,draw=white!15!black},scaled y ticks = false, y tick label style={/pgf/number format/fixed}, legend image post style={scale=.5} 
 ] 
 
 \addlegendentry{Viterbi}; 
\addplot [color=black,solid] table[row sep=crcr]{
1 -39.49659303353572\\
3 -39.49659303353572\\
5 -39.49659303353572\\
7 -39.49659303353572\\
9 -39.49659303353572\\
11 -39.49659303353572\\
13 -39.49659303353572\\
15 -39.49659303353572\\
17 -39.49659303353572\\
19 -39.49659303353572\\
};
 
\addlegendentry{Peek Search}; 
\addplot [color=mycolor1,solid] table[row sep=crcr]{
1  -43.946693546741706 \\
 3  -43.97371100589116 \\
 5 -40.540262010629085 \\
 7 -40.084968164805794 \\
 9 -39.93593788719113 \\
 11 -39.96947072954371\\
 13 -39.84592857565983\\
 15 -39.770266382394196\\
 17 -39.715738794662704\\
 19 -39.65320202389515\\
};

\addlegendentry{OSA};
\addplot [color=mycolor2,solid] table[row sep=crcr]{
1 -45.8590743915589\\
3 -45.8590743915589\\
5 -45.8590743915589\\
7 -45.8590743915589\\
9 -45.8590743915589\\
11 -45.8590743915589\\
13 -45.8590743915589\\
15 -45.8590743915589\\
17 -45.8590743915589\\
19 -45.8590743915589\\
};

\end{axis}
\end{tikzpicture}}
        \label{fig:inclu}
    \end{subfigure}%
    \qquad
    \begin{subfigure}{0.45\linewidth}
    \centering
    \resizebox{\textwidth}{!}{\definecolor{mycolor1}{rgb}{0.00000,0.44700,0.74100}%
 \definecolor{mycolor2}{rgb}{0.85000,0.32500,0.09800} 
 \begin{tikzpicture} 
 \begin{axis}[%
 width= \figW, 
 height= \figH, 
 at={(1.011in,0.642in)}, 
 scale only axis, 
 xmin=1,  
 xmax=19,  
 xlabel={{\bf Latency $L$}},  
 ymin=0.5, 
 ymax=1.0,
 ytick = {0.5, 0.6, 0.7, 0.8, 0.9, 1},
 xtick = {1, 3, 5, 7, 9, 11, 13, 15, 17, 19},
 yticklabel style={/pgf/number format/precision=3}, 
 line width=2pt,  
 grid=both,   
 grid style={line width=.3pt, draw=gray!10},  
 major grid style={line width=.2pt,draw=gray!50}, 
 ylabel={{\bf Fraction of agreement}},   
 axis background/.style={fill=white}, 
 title style= {font=\Large}, 
 title={{\bf Decoding agreement with Viterbi}}, 
 legend style={at={(0.7, 0.6)},
 anchor= south west,legend cell align=left,align=left,draw=white!15!black}, 
 xlabel style={font=\Large},ylabel style={font=\Large}, ticklabel style={font=\Large},y label style={at={(axis description cs:-0.1,0.5)}},x label style={at={(axis description cs:0.5, -0.07)}},legend style={legend cell align=left,align=left,draw=white!15!black},scaled y ticks = false, y tick label style={/pgf/number format/fixed}, legend image post style={scale=.5} 
 ] 
 
\addlegendentry{Peek Search}; 
\addplot [color=mycolor1,solid] table[row sep=crcr]{
1 0.7737957348231927\\
3 0.7697213326974177\\
5 0.8662533215234721\\
7 0.9016556516999387\\
9 0.9135109354772774\\
11 0.9081692443959938\\
13 0.9165360768549431\\
15 0.9303808680248007\\
17 0.9396879471281597\\
19 0.9527423860461948\\
};

\addlegendentry{OSA};
\addplot [color=mycolor2,solid] table[row sep=crcr]{
1 0.5884853852967228\\
3 0.5884853852967228\\
5 0.5884853852967228\\
7 0.5884853852967228\\
9 0.5884853852967228\\
11 0.5884853852967228\\
13 0.5884853852967228\\
15 0.5884853852967228\\
17 0.5884853852967228\\
19 0.5884853852967228\\
};

\end{axis}
\end{tikzpicture}}
        \label{fig:deform}
    \end{subfigure}
    \vskip -0.2in
    \caption{{\bf Evaluation of performance on genome sequence data}. The data consists of 73385 sites, each of which is to be labeled with one of the four states. The log-probability values on the right have been scaled down by a factor of $10^4$ to avoid clutter near the vertical axis. Peek Search achieves almost optimal performance with a latency of only about $20$, which is over three orders of magnitude less than the optimal Viterbi algorithm. The corresponding predictions agreed with the Viterbi algorithm on more than $95\%$ of all sites. In contrast, OSA was found to be significantly suboptimal. 
    \label{fig:genome}}
    \end{figure*}
\section{Experiments} \label{Experiments}
We now describe the results of our experiments that elucidate the practical benefits of our algorithms under low latency constraints.  Specifically, we compare the performance of Peek Search with the state-of-the-art Online Step Algorithm (OSA) \cite{NVS2006} that also provides theoretical guarantees for first order Markov decoding under latency constraints.  OSA hinges on a strong assumption that uncertainty in any state label decreases with increase in latency.  We found that this assumption does not hold in the context of an important application, namely, genome decoding. In contrast, since our algorithms do not make any such assumptions, Peek Search achieves near-optimal performance as expected. 
We experimented with the Glycerol TraSH genome data \cite{DI2013} pertaining to M. tuberculosis transposon mutants. Our task was to label each of the 73385 gene sites with one of the four states, namely essential (ES), growth-defect (GD), non-essential (NE), and growth-advantage (GA). These states represent different categories of gene {\em essentiality}  depending on their read-counts (i.e. emissions), and  the labeling task is crucial toward identifying  potential drug targets for antimicrobial treatment \cite{DI2013}. We used the parameter settings suggested by \cite{DI2013} for decoding with an HMM.  

Note that for this problem, the Viterbi algorithm and heuristics such as beam search need to compute the optimal paths of length equal to the number of sites, i.e. in excess of 73000, thereby incurring very high latency. However, as Fig.  \ref{fig:genome} shows, Peek Search achieved near-optimal log-probability
(the Viterbi objective in \eqref{Viterbi})
with a latency of only about 20, which is less than that of Viterbi by a factor in excess of 3500. Moreover, the state sequence output by Peek Search agreed with the Viterbi labels on more than 95\% of the sites. We observe that, barring downward blips from $L=1$ to $L=3$ and from $L=9$ to $L=11$, the performance improved with $L$. As expected, for all $L$, including those featuring in the blips, the log-probability values were verified to be consistent with our theoretical guarantees. In contrast, we found OSA to be significantly suboptimal in terms of both log-probability and label agreement.\footnote{On a different task, namely, part-of-speech (POS) tagging for sentences in the standard Brown corpus data, we found that both Peek Search and OSA achieved almost optimal performance even with $L=1$.} In particular, OSA  agreed with the optimal algorithm (Viterbi) on only 58.8\% of predictions under both entropy and  expected classification error measures suggested in  \cite{NVS2006}. In fact, just with $L=1$, Peek Search matched with Viterbi on $77.4\%$ predictions thereby outperforming OSA by an overwhelming amount (over $30\%$). Note that Peek Search allows specifying $L$, and thus adapts to the latency directly. However, OSA does not provide a direct control over the latency $L$, and may require extensive tuning\footnote{We varied $\lambda \in \{10^{-4}, 10^{-1}, \ldots, 10^{4}\}$ under both the entropy and the expected classification error measures suggested by \cite{NVS2006} to tune for $L$ (as noted in \cite{NVS2006}, large values of $\lambda$
penalize latency). However, the performance of OSA (as shown in Fig. \ref{fig:genome}) did not improve with increase in $L$.}
of a hyperparameter $\lambda$ to achieve a good trade-off between latency and accuracy. Our empirical findings underscore the promise of our algorithms  toward expediting scientific progress in fields like drug discovery that rely on large-scale genome sequencing.   
\bibliography{main}
\bibliographystyle{unsrt}

\clearpage
\appendix
\section*{Supplementary Material}
We now provide detailed proofs of all the theorems stated in the main text. 

For improved readability,  instead of proving Theorem \ref{Theorem3} immediately, we start with two simpler settings, namely, (a) first order fully connected, and (b) $n^{th}$ order fully connected.  Together with Theorem \ref{TheoremThird}, these results will help segregate the effect of $n$ from that of $\Delta$ on the competitive ratio. 
 
\section{First order chain models with $\Delta = 1$}
\begin{lemma} \label{Theorem1}
The competitive ratio of Peek Search on first order Markov chain models with $\Delta =1$ for $L \geq 1$ is
\begin{eqnarray*} \rho & \leq & \left(1 + \dfrac{1}{L}\right)
\sqrt[\leftroot{-3}\uproot{1} \scriptstyle L ]{L+1}~.
\end{eqnarray*}
\end{lemma}
\begin{proof}
Recall that at each time step $i$, our online algorithm solves the following optimization problem over variables $y \triangleq (y_i, y_{i+1}, \ldots, y_{i+L}) \in S(i, L)$, i.e. the set of valid paths of length $L + 1$ that emanate from the state at time $i$:
$$M_i = \arg\!\!\!\!\max_{y \in S(i, L)} R\left(y_i | \hat{y}_{i-1} \right) + \sum_{j=1}^L \gamma^j R(y_{i+j}|y_{i+j-1}).$$
Note that the set $M$ may have more than one path that maximizes the discounted sum. Breaking ties arbitrarily, let the online algorithm choose $\tilde{y}_i \triangleq (\hat{y}_i, \tilde{y}_i^1, \ldots, \tilde{y}_i^L) \in M_i$ (and reach the state $\hat{y}_i$). Let $\{y_{t}^* ~|~ t \in [T]\}$ be the optimal path over the entire horizon. Since $\Delta = 1$, one of the candidate paths considered by the online algorithm is the optimal segment $(y_i^*, y_{i+1}^*, \ldots, y_{i+L}^*)$. Since $\tilde{y}_i \in M_i$, we must have
\begin{eqnarray} & & R(\hat{y}_i|\hat{y}_{i-1}) ~+~ \gamma R(\tilde{y}_i^1|\hat{y}_{i}) ~+~ \sum_{j=2}^L \gamma^j R(\tilde{y}_{i}^j|\tilde{y}_{i}^{j-1}) \nonumber   \\  
& ~\geq~ & R(y^*_i|\hat{y}_{i-1}) ~+~ \sum_{j=1}^L \gamma^j R(y^*_{i+j}|y^*_{i + j-1}) \nonumber \\ 
& ~\geq~ & \sum_{j=1}^L \gamma^j R(y^*_{i+j}|y^*_{i + j-1})~, \label{eq1}
\end{eqnarray}
where the last inequality follows since all rewards are non-negative, and thus in particular, $R(y^*_i|\hat{y}_{i-1}) \geq 0$.

An alternate path considered by the online algorithm is $(\tilde{y}_{i-1}^1, \ldots, \tilde{y}_{i-1}^L, \bar{y}_{i-1}^{L+1})$, where $(\tilde{y}_{i-1}^1, \ldots, \tilde{y}_{i-1}^L)$ are the last $L$ steps of the path $\tilde{y}_{i-1} \in M_{i-1}$ (i.e. the path chosen at time $i-1$) and $\bar{y}_{i-1}^{L+1}$ is an arbitrary valid transition from state $\tilde{y}_{i-1}^L$. Again since this transition fetches a non-negative reward, we must have
\begin{eqnarray} & & R(\hat{y}_i|\hat{y}_{i-1}) ~+~ \gamma R(\tilde{y}_i^1|\hat{y}_{i}) ~+~ \sum_{j=2}^L \gamma^j R(\tilde{y}_{i}^j|\tilde{y}_{i}^{j-1}) \nonumber   \\  
& ~\geq~ & R(\tilde{y}_{i-1}^1|\hat{y}_{i-1}) ~+~ \sum_{j=1}^{L-1} \gamma^j R(\tilde{y}_{i-1}^{j+1}|\tilde{y}_{i-1}^j)~. \label{eq2} 
\end{eqnarray}
Multiplying \eqref{eq1} by $1-\gamma$ and \eqref{eq2} by $\gamma$, and adding the resulting inequalities, we get
\begin{eqnarray}
& & R(\hat{y}_i|\hat{y}_{i-1}) ~+~ \gamma R(\tilde{y}_i^1|\hat{y}_{i}) ~+~ \sum_{j=2}^L \gamma^j R(\tilde{y}_{i}^j|\tilde{y}_{i}^{j-1}) \nonumber   \\  
& ~\geq~ & \sum_{j=1}^L (1 - \gamma) \gamma^j R(y^*_{i+j}|y^*_{i + j-1}) ~+~ \nonumber 
 ~+~  \gamma R(\tilde{y}_{i-1}^1|\hat{y}_{i-1}) 
  ~+~  \sum_{j=1}^{L-1} \gamma^{j+1}  R(\tilde{y}_{i-1}^{j+1}|\tilde{y}_{i-1}^j) \nonumber \\
& = & \sum_{j=1}^L (1 - \gamma) \gamma^j R(y^*_{i+j}|y^*_{i + j-1}) \nonumber 
 ~+~  \gamma R(\tilde{y}_{i-1}^1|\hat{y}_{i-1}) ~+~ \sum_{k=2}^{L} \gamma^{k} R(\tilde{y}_{i-1}^{k}|\tilde{y}_{i-1}^{k-1}),
\label{eq3}
\end{eqnarray}
where the last inequality follows from a change of variable, namely, $k = j+1$. 
Summing across all time steps $i$, 
\begin{eqnarray*}
\sum_i R(\hat{y}_i|\hat{y}_{i-1}) + \underbrace{\sum_i \left(\gamma R(\tilde{y}_i^1|\hat{y}_{i}) ~+~ \sum_{j=2}^L \gamma^j R(\tilde{y}_{i}^j|\tilde{y}_{i}^{j-1}) \right)}_{DR1} \nonumber
\end{eqnarray*}
\begin{eqnarray}
& ~\geq~ & \sum_i \sum_{j=1}^L (1 - \gamma) \gamma^j R(y^*_{i+j}|y^*_{i+j-1}) \nonumber
 ~+~  \underbrace{\sum_i \left(\gamma R(\tilde{y}_{i-1}^1|\hat{y}_{i-1}) ~+~ \sum_{j=2}^{L} \gamma^{j} R(\tilde{y}_{i-1}^{j}|\tilde{y}_{i-1}^{j-1})\right)}_{DR2} ~. \nonumber
\end{eqnarray}
Without loss of generality, we can assume that all transitions between states in the first $L+1$ time steps and the last $L+1$ steps fetch zero reward.\footnote{One way to accomplish this is by adding a sequence of  $L+1$ dummy tokens at the beginning and another sequence at the end of the input to be decoded. Alternatively, we can introduce a dummy start state that transitions to itself $L$ times with zero reward and produces a fake output in each transition, and then makes a zero reward transition into the true start state, whence actual decoding happens for $T$ steps followed by repeated transitions into a dummy end state that again fetches zero reward).} Now note that both $DR1$ and $DR2$ consist of terms that pertain to leftover discounted rewards on optimal $(L+1)$-paths computed by Peek Search (recall we take only the first step on each such path). In fact, the terms are common to both sides except for those that fall in $(L+1)$-length windows at the beginning or the end. Since first and last $(L+1)$ steps fetch zero reward, we can safely disregard these windows. Thus, by telescoping over $i$, we have
$$\sum_i R(\hat{y}_i|\hat{y}_{i-1}) ~~\geq~~ \sum_i \sum_{j=1}^L (1 - \gamma) \gamma^j R(y^*_{i+j}|y^*_{i+j-1})~.$$
Defining a variable $s = i+j$, and interchanging the two sums, we note that the right side becomes 
\begin{eqnarray*}
(1 - \gamma) \sum_{j=1}^L \gamma^j \sum_{s} R(y^*_s|y^*_{s-1}).    
\end{eqnarray*}
That is, every reward subsequent to $L+1$ steps appears with discounts $\gamma, \gamma^2, \ldots, \gamma^L$. Summing the geometric series, we note that the ratio of the total reward obtained by the optimal offline algorithm to that by the online algorithm, i.e. the competitive ratio $\rho$ is at most $\gamma^{-1} (1 - \gamma^L)^{-1}$. The result follows by setting $\gamma = \sqrt[\leftroot{-3}\uproot{1} \scriptstyle L ]{1/(L+1)}$. 
\end{proof}

\section{$n^{th}$ order chain models with $\Delta = 1$}
\begin{lemma} \label{Theorem2}
The competitive ratio of Peek Search on Markov chain models of order $n$ with $\Delta = 1$ for $L \geq n$ is
\begin{eqnarray*} \rho & ~\leq~   \dfrac{L+1}{L-n+1} \left(\dfrac{L+1}{n}\right)^{n/(L-n+1)}
& ~=~  1 + \Theta \left(\dfrac{\log L}{L - n + 1} \right)~.
\end{eqnarray*}
\end{lemma}
\begin{proof}
For $n = 1$, the result follows from Lemma \ref{Theorem1}. Therefore, we will assume $n > 1$. The online algorithm finds, at time $i$, some $\tilde{y}_i \triangleq (\hat{y}_i, \tilde{y}_i^1, \ldots, \tilde{y}_i^L)$ that maximizes the following objective over valid paths $y = (y_i, \ldots, y_{i+L})$:
\begin{eqnarray*} f(y)  \triangleq  R(y_i|\hat{y}_{[i-n, i-1]}) 
& \quad+~ & \sum_{j=1}^{n-1} \gamma^j R(y_{i+j}|\hat{y}_{[i-n+j, i-1]}, y_{[i, i+j-1]}) \\ 
& \qquad+~ & \sum_{j=n}^L \gamma^j R(y_{i+j}|y_{[i+j-n, i+j-1]})~. \end{eqnarray*}
One candidate path for the online algorithm (a) makes a transition to $y_i^*$ worth $R(y_i^*|\hat{y}_{[i-n, i-1]}) \geq 0$, (b) then follows the sequence of $n-1$ states $y^*_{[i+1, i+n-1]}$ where transition $i+j, j \in [n-1]$ is worth $$\gamma^j R(y^*_{i+j}|\hat{y}_{[i-n+j, i-1]}, y^*_{[i, i+j-1]}) \geq 0~,$$ and (c) finally follows a sequence of $L-n+1$ states $y^*_{[i+n, i+L]}$ where transition $i+j, j \in \{n, n+1, \ldots, L\}$ is worth $\gamma^j R(y^*_{i+j}|y^*_{[i+j-n, i+j-1]})~.$  
Since $\tilde{y}_i \in \arg\!\max_y f(y)$ and the rewards in (a) and (b) are all non-negative, we must have 
\begin{equation} \label{Beq1} f(\tilde{y}_i) \geq \sum_{j=n}^L \gamma^j R(y^*_{i+j}|y^*_{[i+j-n, i+j-1]})~.\end{equation}

Another option available with the online algorithm is to continue following the path selected at time $i-1$ for $L$ steps, and then make an additional arbitrary transition with a non-negative reward. Therefore, we must also have
\begin{eqnarray}
f(\tilde{y}_i)  ~\geq~  R(\tilde{y}_{i-1}^1|\hat{y}_{[i-n, i-1]}) 
& ~+~ &  \sum_{j=1}^{n-1} \gamma^j R(\tilde{y}_{i-1}^{j+1}|\hat{y}_{[i-n+j, i-1]}, \tilde{y}_{i-1}^{[j]}) \nonumber \\
& ~+~ &  \sum_{j=n}^{L-1} \gamma^j R(\tilde{y}_{i-1}^{j+1}|\tilde{y}_{i-1}^{[j-n+1, j]}) \label{Beq2} ~.
\end{eqnarray}
Multiplying \eqref{Beq1} by $1-\gamma$ and \eqref{Beq2} by $\gamma$, and adding the resulting inequalities, we get
\begin{eqnarray}
f(\tilde{y}_i) & ~\geq~ & (1-\gamma) \sum_{j=n}^L \gamma^j R(y_{i+j}^*|y^*_{[i+j-n, i+j-1]}) 
 ~+~  \gamma R(\tilde{y}_{i-1}^1|\hat{y}_{[i-n, i-1]}) \nonumber \\
& ~+~ & \sum_{j=2}^n \gamma^j R(\tilde{y}_{i-1}^j|\hat{y}_{[i-n+j-1, i-1]}, \tilde{y}_{i-1}^{[j-1]})  
 ~+~  \sum_{j=n+1}^L \gamma^j R(\tilde{y}_{i-1}^j|\tilde{y}_{i-1}^{[j-n, j-1]})~. \label{Beq3}
\end{eqnarray}
Expanding the terms of $f(\tilde{y}_i)$, we note 
\begin{eqnarray} 
f(\tilde{y}_i) & ~=~ & R(\hat{y}_i|\hat{y}_{[i-n, i-1]}) ~+~ \gamma R(\tilde{y}_i^1|\hat{y}_{[i-n+1, i]}) \nonumber\\
& \qquad+~ & \sum_{j=2}^n \gamma^j R(\tilde{y}_i^j|\hat{y}_{[i+j-n, i]}, \tilde{y}_i^{[j-1]}) 
 \qquad+~ \sum_{j=n+1}^L \gamma^j R(\tilde{y}_i^j|\tilde{y}_i^{[j-n, j-1]})~. \label{Beq4}
\end{eqnarray}
Substituting $f(\tilde{y}_i)$ from \eqref{Beq4} in \eqref{Beq3}, assuming zero padding as in the proof of Lemma \ref{Theorem1}, and summing over all time steps $i$, we get the inequality
\begin{eqnarray*} 
& \displaystyle \sum_i R(\hat{y}_i|\hat{y}_{[i-n, i-1]}) 
& ~\geq~ \displaystyle \sum_i \sum_{j=n}^L (1-\gamma)  \gamma^j R(y^*_{i+j}|y^*_{[i+j-n, i+j-1]})~.
\end{eqnarray*}
Defining $s=i+j$ and interchanging the two sums, we note that the right side simplifies to 
$$(1- \gamma) \sum_{j=n}^L \gamma^j \sum_s R(y_s^*|y^*_{[s-n, s-1]})~.$$
The sum of this geometric series is given by $\gamma^n - \gamma^{L+1}$, and thus setting $$\gamma = \left(\dfrac{n}{L+1}\right)^{1/(L-n+1)}~,$$ we immediately conclude that the total reward obtained by the optimal offline algorithm exceeds that of the online algorithm by at most $\Theta\left(\dfrac{\log L}{L-n+1}\right)$ times the reward of the online algorithm, and hence we have the following bound on the competitive ratio
$$\rho \leq 1 + \Theta \left(\dfrac{\log L}{L - n + 1} \right)~.$$
\end{proof}

We are now ready to prove Theorem \ref{Theorem3}.

\section{$n^{th}$ order chain models with diameter $\Delta$}
\begin{theorem} \label{Theorem3}
The competitive ratio of Peek Search on Markov chain models of order $n \geq 1$ with diameter $\Delta \geq 1$ for  $L \geq \Delta + n - 1$,
\begin{eqnarray*} \rho & \leq & \dfrac{L+1}{L-\Delta-n+2} \left(\dfrac{L+1}{\Delta+n-1}\right)^{(n + \Delta - 1)/(L - \Delta - n + 2)} \\
& = & 1 + \Theta \left(\dfrac{\log L}{L - \Delta - n + 2} \right)~.
\end{eqnarray*}
\end{theorem}
\begin{proof}
For $\Delta = 1$, the result follows from Lemma \ref{Theorem2}. Therefore, we will assume $\Delta > 1$.  As in the proof of Theorem \ref{Theorem2}, the online algorithm finds at time $i$ some $\tilde{y}_i \triangleq (\hat{y}_i, \tilde{y}_i^1, \ldots, \tilde{y}_i^L)$ that maximizes the following objective over valid paths $y = (y_i, \ldots, y_{i+L})$:
\begin{eqnarray*} f(y)  ~~\triangleq~~  R(y_i|\hat{y}_{[i-n, i-1]}) 
& \quad+~ &  \sum_{j=1}^{n-1} \gamma^j R(y_{i+j}|\hat{y}_{[i-n+j, i-1]}, y_{[i, i+j-1]}) \\ 
& \quad+~ & \sum_{j=n}^L \gamma^j R(y_{i+j}|y_{[i+j-n, i+j-1]})~. \end{eqnarray*}
Since $\Delta > 1$, the online algorithm may not be able to jump to the desired state on the optimal offline path in one step unlike in the setting of Lemma \ref{Theorem2}, and may require $\Delta$ steps in the worst case.\footnote{The online algorithm may require less than $\Delta$ steps depending on its current state, however, we perform a worst case analysis and therefore, our result holds even  if fewer than $\Delta$ steps may suffice to reach the optimal path at some point during the execution of the online algorithm.} Therefore, let $(\bar{y}_i, \ldots, \bar{y}_{i+\Delta-2})$ be an intermediate sequence of states before the online algorithm could transition to the optimal offline path and then follow the optimal algorithm for the remaining steps. Therefore, we have 
\begin{eqnarray}
 f(\tilde{y}_i) ~~\geq~~ R(\bar{y}_i|\hat{y}_{[i-n, i-1]}) 
& + & \sum_{j=1}^{\Delta-2} \gamma^j R(\bar{y}_{i+j}|\hat{y}_{[i-n+j, i-1]}, \bar{y}_{[i, i+j-1]}) \nonumber\\
& + & \gamma^{\Delta-1} R(y^*_{i+\Delta-1}|\bar{y}_{[i+\Delta-n-1, i+\Delta-2]}) \nonumber \\
& + & \hspace{-0.3cm}\sum_{j=\Delta}^{\Delta+n-2}\hspace{-0.1cm} \gamma^j R(y^*_{i+j}|\bar{y}_{[i+j-n-1, i+\Delta-2]}, y^*_{[i+\Delta-1, i+j-1]}) \nonumber \\
& + & \sum_{j=\Delta+n-1}^{L} \gamma^j R(y^*_{i+j}|y^*_{[i+j-n, i+j-1]}) \nonumber \\
  && \hspace*{-4cm}  ~\geq~ \sum_{j=\Delta+n-1}^{L} \gamma^j R(y^*_{i+j}|y^*_{[i+j-n, i+j-1]})~, \label{Ceq1}
\end{eqnarray}
where we have leveraged the non-negativity of rewards to obtain the last inequality. 

Another option available with the online algorithm is to continue following the path selected at time $i-1$ for $L$ steps, and then make an additional arbitrary transition with a non-negative reward. Therefore, we must also have
\begin{eqnarray}
f(\tilde{y}_i)  ~\geq~  R(\tilde{y}_{i-1}^1|\hat{y}_{[i-n, i-1]}) 
& ~+~ & \sum_{j=1}^{n-1} \gamma^j R(\tilde{y}_{i-1}^{j+1}|\hat{y}_{[i-n+j, i-1]}, \tilde{y}_{i-1}^{[j]}) \nonumber \\
& ~+~ & \sum_{j=n}^{L-1} \gamma^j R(\tilde{y}_{i-1}^{j+1}|\tilde{y}_{i-1}^{[j-n+1, j]}) \label{Ceq2} ~.
\end{eqnarray}
Multiplying \eqref{Ceq1} by $1-\gamma$ and \eqref{Ceq2} by $\gamma$, and adding the resulting inequalities, we get
\begin{eqnarray*}
f(\tilde{y}_i) & ~\geq~ & (1-\gamma) \sum_{j=\Delta+n-1}^L \gamma^j R(y_{i+j}^*|y^*_{[i+j-n, i+j-1]})
 ~+~  \gamma R(\tilde{y}_{i-1}^1|\hat{y}_{[i-n, i-1]}) \\
& ~+~ & \sum_{j=1}^{n-1} \gamma^{j+1} R(\tilde{y}_{i-1}^{j+1}|\hat{y}_{[i-n+j, i-1]}), \tilde{y}_{i-1}^{[j]})
 ~+~  \sum_{j=n}^{L-1}\gamma^{j+1}R(\tilde{y}_{i-1}^{j+1}|\tilde{y}_{i-1}^{[j-n+1, j]})~.
\end{eqnarray*}
Expanding $f(\tilde{y}_i)$, telescoping over $i$, and defining $s = i+j$ as in Lemma \ref{Theorem2}, we get that the total reward accumulated by the online algorithm is at least
$$(1-\gamma) \sum_{j=\Delta+n-1}^L \gamma^j \quad=~~ (\gamma^{n+\Delta-1}-\gamma^{L+1})$$
times the total reward collected by the optimal offline algorithm. Setting $$\gamma =  \sqrt[\leftroot{-3}\uproot{3} \scriptstyle (L - \Delta - n + 2)]{\dfrac{\Delta+n-1}{L+1}}~,$$ we immediately get
\begin{eqnarray*} \rho & \leq &  \dfrac{L+1}{L-\Delta-n+2} \left(\dfrac{L+1}{\Delta+n-1}\right)^{(n + \Delta - 1)/(L - \Delta - n + 2)} 
\\
& = & 1 + \Theta \left(\dfrac{\log L}{L - \Delta - n + 2} \right)~.
\end{eqnarray*}
\end{proof}
Note that Theorem \ref{Theorem3} suggests that essentially $n+\Delta-1$ steps are wasted every $L+1$ steps by the online algorithm in the sense that it may not receive any reward in these steps. However, the remaining steps fetch nearly the same reward as the optimal offline algorithm. In particular, the competitive ration $\rho$ gets arbitrarily close to 1, as $L$ is set sufficiently large compared to $\Delta+n$. That is, the performance of the online algorithm is asymptotically optimal in the peek $L$.

We now show that the result extends to the non-homogeneous setting. 

\section{Non-homogeneous Markov chain models}
We note that there might be multiple transitions between a pair of states during any peek window. Such transitions are considered distinct and may indeed have different rewards during the same window. We only require that the non-discounted reward committed for every transition is ``honored" at all times during the window. We have the following result.

{\em The competitive ratio of Peek Search on non-homogeneous (i.e. time-varying) Markov chain models of order $n$ with diameter $\Delta$~ for  $L \geq \Delta + n - 1$ is
\begin{eqnarray*} \rho & \leq & \dfrac{L+1}{L-\Delta-n+2} \left(\dfrac{L+1}{\Delta+n-1}\right)^{(n + \Delta - 1)/(L - \Delta - n + 2)} \\
& = & 1 + \Theta \left(\dfrac{\log L}{L - \Delta - n + 2} \right)~,
\end{eqnarray*}
provided the reward associated with any transition does not change for (at least) $L+1$ steps from the time it is revealed as peek information to the online algorithm.}   

\begin{proof}
The online algorithm maximizes the following non-stationary objective at time $i$:
\begin{eqnarray*} f_i(y)  ~~~\triangleq~~~  R_i(y_i|\hat{y}_{[i-n, i-1]}) 
& ~+~ &  \sum_{j=1}^{n-1} \gamma^j R_i(y_{i+j}|\hat{y}_{[i-n+j, i-1]}, y_{[i, i+j-1]}) \\
& ~+~ &  \sum_{j=n}^L \gamma^j R_i(y_{i+j}|y_{[i+j-n, i+j-1]})~, \end{eqnarray*}
where the subscript $i$ shown with $f$ and $R$ indicates that the rewards associated with a transition may change with time $i$. Proceeding as in the proof of Theorem \ref{Theorem3}, we get 
\begin{eqnarray*}
f_i(\tilde{y}_i) & \geq & (1-\gamma) \sum_{j=\Delta+n-1}^L \gamma^j R_i(y_{i+j}^*|y^*_{[i+j-n, i+j-1]})\\
& \qquad+ & \gamma R_i(\tilde{y}_{i-1}^1|\hat{y}_{[i-n, i-1]}) \\
& \qquad+ & \sum_{j=1}^{n-1} \gamma^{j+1} R_i(\tilde{y}_{i-1}^{j+1}|\hat{y}_{[i-n+j, i-1]}), \tilde{y}_{i-1}^{[j]})\\
& \qquad+ & \sum_{j=n}^{L-1}\gamma^{j+1}R_i(\tilde{y}_{i-1}^{j+1}|\tilde{y}_{i-1}^{[j-n+1, j]})~.
\end{eqnarray*}
However, by our assumption, we can equivalently write
\begin{eqnarray*}
f_i(\tilde{y}_i) & \geq & (1-\gamma) \sum_{j=\Delta+n-1}^L \gamma^j R_i(y_{i+j}^*|y^*_{[i+j-n, i+j-1]})\\
& \qquad+ & \gamma R_{i-1}(\tilde{y}_{i-1}^1|\hat{y}_{[i-n, i-1]}) \\
& \qquad+ & \sum_{j=1}^{n-1} \gamma^{j+1} R_{i-1}(\tilde{y}_{i-1}^{j+1}|\hat{y}_{[i-n+j, i-1]}), \tilde{y}_{i-1}^{[j]})\\
& \qquad+ & \sum_{j=n}^{L-1}\gamma^{j+1}R_{i-1}(\tilde{y}_{i-1}^{j+1}|\tilde{y}_{i-1}^{[j-n+1, j]})~.
\end{eqnarray*}
Expanding $f(\tilde{y}_i)$, summing over all $i$, and defining $s = i+j$ as in Theorem \ref{Theorem2}, we get
\begin{eqnarray*} 
 \displaystyle \sum_i R_i(\hat{y}_i|\hat{y}_{[i-n, i-1]}) 
 & ~\geq~ & \displaystyle \sum_i \sum_{j=\Delta+n-1}^L (1-\gamma)  \gamma^j R_i(y^*_{i+j}|y^*_{[i+j-n, i+j-1]})\\
& ~=~ & (1-\gamma) \displaystyle \sum_{j=\Delta+n-1}^L \gamma^j \sum_s  R_{s-j}(y^*_s|y^*_{[s-n, s-1]})\\
& ~=~ & (1-\gamma) \displaystyle \sum_{j=\Delta+n-1}^L \gamma^j \sum_s  R_{s}(y^*_s|y^*_{[s-n, s-1]})~,
\end{eqnarray*}
where we have again made use of the fact that reward due to any transition does not change for $L+1$ steps once revealed.  The rest of the proof is identical to the analysis near the end of proof for Theorem \ref{Theorem3}. 
\end{proof}

\newpage
\section{Efficient Dynamic Programs}
\begin{theorem} \label{TheoremDP}
Peek Search can compute a best $\gamma$-discounted path for the next $L+1$ steps, in $n^{th}$ order Markov chain models, in time $O(L|K|^{n})$, where $K$ is the set of states.
\end{theorem}
\begin{proof}
Let $S_i(\ell, v_{[a,b]})$ denote the set of all valid paths of length $\ell+1$ emanating from the state $\hat{y}_{i-1}$ at time $i$, where $\ell \in \{0, 1, \ldots, L\}$, that end in the state sequence $(v_a, \ldots, v_b)$. Thus, e.g., if the directed edge $e = (\hat{y}_{i-1}, v_n)$ exists, then 
\begin{eqnarray*}
S_i(0, v[2, n]) & = ~~ 
\begin{cases}
\{e\} & \mbox{ if } v_{n-j} = \hat{y}_{i-j},~\forall j \in [ n-2]\\
\emptyset & \mbox{ otherwise },
\end{cases}
\end{eqnarray*}
where $\emptyset$ is the empty set. We also denote the reward resulting from valid paths of length $\ell+1$ that end in sequence $v_{[a, b]}$ by $\Pi_i(\ell, v[a, b])$. That is, 
\begin{eqnarray*}
\Pi_i(\ell, v_{[a, b]})  = \max_{(y_i, \ldots, y_{i+\ell}) \in S_i(\ell, v_{[a, b]})} f_{\ell}(y_{[i,i+\ell]}),
\end{eqnarray*}
where we define $f_{\ell}(y_{[i,i+\ell]})$ recursively as
\begin{eqnarray*}
f_{\ell}(y_{[i,i+\ell]}) & = &
\begin{cases}
R(y_i|\hat{y}_{[i-n, i-1]}) \hspace{6.1cm}  \ell = 0 \\ \\
f_{\ell-1}(y_{[i, i+\ell-1]}) ~+~ \gamma^{\ell} R(y_{i+\ell}|\hat{y}_{[i-n+\ell, i-1]}, y_{[i, i+\ell-1]}) \qquad \ell \in [n-1] \\ \\ 
f_{\ell-1}(y_{[i, i+\ell-1]}) ~+~ \gamma^{\ell} R(y_{i+\ell}|y_{[i-n+\ell, i+\ell-1]}) \qquad \qquad ~~~~~ \ell \in [n, L]
\end{cases}~~.
\end{eqnarray*}
Note that $f_L(y_{i, i+L})$ is precisely the objective optimized by Peek Search at time $i$. Now, suppose $\ell \in [n, L]$. Then, for any end sequence $v_{[2, n]}$, 
\begin{eqnarray*}
\Pi_i(\ell, v_{[2, n]}) & = &  \max_{y_{[i,i+\ell] \in S_i(\ell, v_{[2,n]})}} f_{\ell}(y_{[i,i+\ell]})\\
& = & \max_{v_1} \max_{y_{[i,i+\ell] \in S_i(\ell, v_{[1,n]})}} f_{\ell}(y_{[i,i+\ell]})~,
\end{eqnarray*}
which may be expanded as\footnote{We simply write $S_i(\ell, v)$ instead of $y_{[i,i+\ell]} \in S_i(\ell, v)$ in order to improve readability at the expense of abuse of notation.}
\begin{eqnarray*}
 &&  \max_{v_1 \in K} \max_{y_{[i,i+\ell] \in S_i(\ell, v_{[1,n]})}} f_{\ell-1}(y_{[i, i+\ell-1]}) ~+~ \gamma^{\ell} R(y_{i+\ell}|y_{[i-n+\ell, i+\ell-1]})\\
& =~~ &  \max_{v_1} \max_{S_i(\ell, v_{[n]})} f_{\ell-1}(y_{[i, i+\ell-1]}) ~+~ \gamma^{\ell} R(v_n|v_{[n-1]})\\
& =~~ &   \max_{v_1} \max_{S_i(\ell-1, v_{[n-1]})} f_{\ell-1}(y_{[i, i+\ell-1]}) ~+~ \gamma^{\ell} R(v_n|v_{[n-1]})\\
& =~~ & \max_{v_1}~~ \left(\Pi_i(\ell-1, v_{[n-1]}) ~+~ \gamma^{\ell} R(v_n|v_{[n-1]})\right)~.
\end{eqnarray*}
A similar analysis can be done for $\ell \in [n-1]$. Then, the maximizing path of length $\ell+1$ is in the set
$$\arg\!\!\max_{v_{[2, n]} \in K}\max_{v_1 \in K}~~ \left(\Pi_i(\ell-1, v_{[n-1]}) ~+~ \gamma^{\ell} R(v_n|v_{[n-1]})\right),$$
which requires\footnote{In addition to {\em backpointer} information that is required to determine a maximizing path as in the Viterbi algorithm once the construction of table for bookkeeping $\Pi_i$ is completed. Construction of table requires $O(L|K|^n)$ time which dominates the $O(L)$ time required for computing the path from the backpointers.} checking $O(|K|^n)$ values for $v_{[n]}$. We conclude by noting that $\Pi_i$ is updated for each $\ell \in \{0, \ldots, L\}$, and thus the total complexity is $O(L|K|^n)$. 

We sketch our efficient traceback procedure in Algorithm \ref{alg:peeksearch}. In the procedure, we let $S_i^{(\ell)}, \ell \in \{0, \ldots, L\}$ be all state sequences of length $\ell+1$ that start from state at time $i$. Thus, for instance, $S_i^{(0)}$ contains all states $y_i$ that can be reached in one step.  

\begin{algorithm}[tb]
  \caption{Peek Search ($\gamma, L, R_i, \hat{y}_{i-n}, \ldots, \hat{y}_{i-1})$}
  \label{alg:peeksearch}
\begin{algorithmic}
  \STATE {\bfseries Input:} previous states $\hat{y}_{[i-n, i-2]}$ and current state $\hat{y}_{i-1}$, latency $L$, discount factor $\gamma$ and reward function $R_i(\cdot|\cdot)$ 
  \STATE {\bfseries Output:} a sequence of states that maximizes the $\gamma$-discounted reward over paths of length $(L+1)$ \\
  \underline{\em Initialize rewards available in the immediate step}\\
  \STATE Set $y_{i-j} = \hat{y}_{i-j}, ~~~ \forall j \in [n]$  
 \STATE $$\Pi_i(0, y_{[i-n, i-1]}, y_i) = \begin{cases}
  R_i(y_i|y_{[i-n, i-1}]), y_i \in S_i^{(0)}\\
  0 \qquad \qquad \qquad ~~ \mbox{ otherwise}
  \end{cases} $$
  \underline{\em Update rewards \& backpointers incrementally}
  \STATE Define the shorthand~~~ $y_{(a, b)}^{m, n} \triangleq y_{[a+m, b+n]} $ 
  \FOR{$\ell=1$ {\bfseries to} $L$ for $y_{i+\ell} \in S_i^{(\ell)}$}
  \STATE \begin{eqnarray*} \Pi_i(\ell, y_{(i, i-1)}^{\ell-n, \ell}, y_{i+\ell})
 \hspace*{-0.4cm}  & ~=  & \max_z \bigg( \Pi_i(\ell-1, z, y_{(i, i-1)}^{\ell-n, \ell}) ~+~   \gamma^{\ell} R_i(y_{i+\ell}|z,   y_{(i, i-1)}^{\ell-n, \ell})\bigg)
  \end{eqnarray*}
\STATE Store the backpointer 
$z^*_{\ell}(y_{i+\ell})$ that maximizes the score $\Pi_i(\ell, y_{(i, i-1)}^{\ell-n, \ell}, y_{i+\ell})$ above
\ENDFOR
\STATE \underline{\em Trace back a path with maximum discounted reward}
\STATE $\tilde{y}_{i+L} \in \displaystyle \arg\!\max_{y_{i+L}} \max_{y_{[i+L-n, i-1+L]}} \Pi_i(L, y_{(i, i-1)}^{L-n, L}, y_{i+L+1}) $

\FOR{$\ell=L-1$ {\bfseries to} 0}
\STATE $\tilde{y}_{i+\ell} = z^*_{\ell+1}(\tilde{y}_{i+\ell+1})$
\ENDFOR
\STATE Set $\hat{y}_i = \tilde{y}_i$
\end{algorithmic}
\end{algorithm}

Note that both Randomized Peek Search and Peek Reset, can compute rewards on their  paths efficiently by using our procedure for Peek Search as a subroutine. For instance, Randomized Peek Search could invoke Algorithm \ref{alg:peeksearch} at each reset point with $\gamma$ set to 1, and follow this path until the next reset point.
\end{proof}

\section{Randomized Peek Search}
\begin{theorem} \label{Theorem7}
Randomized Peek Search achieves, in expectation, on Markov chain models of order $n$ with diameter $\Delta$ a competitive ratio
\begin{eqnarray*}\rho  & \leq &   1 + \dfrac{\Delta + n - 1}{L+1 - (\Delta + n - 1)} \\
 & = &  1 + \Theta \left(\dfrac{1}{L - \tilde{\Delta} + 1} \right)~.
\end{eqnarray*} 
\end{theorem}
\begin{proof}
Recall that the randomized algorithm recomputes and follows a path that optimizes the non-discounted reward once every $L+1$ steps (which we call an {\em epoch}). Since the starting or reset point is chosen uniformly at random from $\{1, 2, \ldots, L+1\}$, we define a  random variable $X$ that denotes the outcome of an unbiased $(L+1)$-sided dice. Let $(X=x)$ be any particular realization.  Then, during epoch $i$, one option available with the online algorithm is to give up rewards in steps $$[i*(L+1)+x, i*(L+1)+x+\Delta+n-2]$$ to reach a state on the optimal offline path and follow it for the remainder of the epoch. Let $ON_{x}$ denote the total reward of the online randomized algorithm conditioned on realization $x$, and let $OPT$ be the optimal reward. 
Then, letting $r_t^*$ be the reward obtained by the optimal offline algorithm at time $t$  we must have
\begin{equation} \label{Geq1}
ON_x ~~~ \geq ~~~  OPT ~-~ \sum_i ~~ \sum_{t=i*(L+1)+x}^{i*(L+1)+x+\Delta+n-2} r_t^*~. \end{equation}

Since $x$ is chosen uniformly at random from $[L+1]$, we also note the expected value of the second term on the right
\begin{eqnarray*}
& = & \mathbb{E}_x \left(\sum_{i}~~ \sum_{t=i*(L+1)+x}^{i*(L+1)+x+\Delta+n-2} r_t^* \bigg| X=x \right)
\\
& = & \dfrac{1}{L+1} \sum_{x=1}^{L+1}~~\sum_{i}~~ \sum_{t=i*(L+1)+x}^{i*(L+1)+x+\Delta+n-2} r_t^* \\
& = & \dfrac{1}{L+1} \sum_{x=1}^{L+1}~~\sum_{i}~~ \sum_{z=0}^{\Delta+n-2} r_{z+i*(L+1)+x}^* \\
& = & \dfrac{1}{L+1} \sum_{z=0}^{\Delta+n-2} \left(\sum_{i} ~~ \sum_{x=1}^{L+1} ~~ r_{z+i*(L+1)+x}^* \right)~\\
& = & \dfrac{1}{L+1} \sum_{z=0}^{\Delta+n-2} OPT\\
& = & \dfrac{\Delta+n-1}{L+1} ~~OPT~~.
\end{eqnarray*}

Therefore, taking expectations on both sides of \eqref{Geq1}, 
$$\mathbb{E}_x (ON_x) \geq OPT \left(1 - \dfrac{\Delta+n-1}{L+1}\right), $$
whence the result follows immediately.

\end{proof}

\section{Peek Reset}
\begin{theorem} \label{Theorem6}
The competitive ratio of Peek Reset on Markov chain models of order $n$ with diameter $\Delta$~ for latency $L$ is 
\begin{eqnarray*} \rho & ~~\leq~~  1 ~+~ \dfrac{2(\Delta+n)(\Delta+n-1)}{L-8(\Delta+n-1)+1} 
& ~~=~~ 1 + \Theta \left(\dfrac{1}{L - 8\tilde{\Delta} + 1} \right)~.
\end{eqnarray*}
\end{theorem}
\begin{proof}
We will assume for now that $L$ is a multiple of $4(\Delta+n-1)$. Recall that the Peek Reset algorithm works in phases with varying lengths, and takes multiple steps in each phase. Let $(i)$ denote the time at which phase $i$ begins. Then, the algorithm follows, in phase $i$, a sequence of states $\hat{y}(i) \triangleq (\hat{y}_{(i)}, \hat{y}_{(i)+1}, \ldots, \hat{y}_{T_i-1})$ that maximizes the following objective over valid paths $y = (y_{(i)}, \ldots, y_{T_i-1})$ :
\begin{eqnarray*}
f(y)  & \triangleq & R(y_{(i)}|\hat{y}_{[(i)-n, (i)-1]}) \\
& \qquad + & \sum_{j=1}^{n-1} R(y_{(i)+j}|\hat{y}_{[(i)-n+j, (i)-1]}, y_{[(i), (i)+j-1]}) \\
& \qquad + & \sum_{j=n}^{T_i-(i)-1} R(y_{(i)+j}|y_{[(i)+j-n, (i)+j-1]})~,
\end{eqnarray*}
where $T_i$ is chosen arbitrarily from the set
$$\arg\!\!\!\!\!\!\!\!\!\!\!\!\!\min_{t \in [(i) + L/2 + 1], (i)+L]} \max_{(y_{t-n}, \ldots, y_t)} R(y_t|y_{[t-n, t-1]}) ~.$$
We define the corresponding reward
$$x_{T_i} =   \min_{t \in [(i) + L/2 + 1], (i)+L]} \max_{(y_{t-n}, \ldots, y_t)} R(y_t|y_{[t-n, t-1]}) ~.$$

Consider the portion of the path traced by the online algorithm from $\hat{y}_{(i)+L/2}$ to $\hat{y}_{T_i-1}$. Total number of edges on this path is $z_i = T_i - ((i)+L/2+1)$. We claim that the reward resulting from this sequence is at least
\begin{equation*} 
a_i = \dfrac{z_i - (\Delta+n-1)}{\Delta+n}~x_{T_i}~.
\end{equation*}
This is true since, by definition of $x_{T_i}$, at each time $t \in [(i) + L/2 + 1, (i)+L]$, there is a state $y_{t-1}$ such that moving to some state $y_t$ will fetch a reward at least $x_{T_i}$. Note that a maximum of $\Delta + n-1$ steps might have to be wasted to reach another state that fetches at least $x_{T_i}$. Thus, a reward of $x_{T_i}$ is guaranteed every $\Delta+n$ steps. While there are $z_i$ steps in this sequence, at most $\Delta+n-1$ steps may be left over as residual edges that do not fetch any reward if $z_i$ is not a multiple of $\Delta+n$. Since the online algorithm optimized for total non-discounted reward, it must have considered this alternative subsequence of steps for the interval pertaining to $z_i$. 

Next consider the portion traversed by the online algorithm from $\hat{y}_{T_i}$ to $\hat{y}_{(i)+L}$ in the next phase $(i+1)$. This phase starts at time $T_i$. By an argument analogous to previous paragraph, the online algorithm collects from this sequence an aggregate no less than
$$b_i = \dfrac{(i)+L-T_i - (\Delta+n-1)}{\Delta+n}~x_{T_i}~.$$
Thus, the reward accumulated by the online algorithm in these two portions is at least
\begin{equation*} 
a_i +  b_i = \dfrac{L - 4(\Delta+n-1)}{2(\Delta+n)}~ x_{T_i}~.
\end{equation*} 
Summing over all phases, we note that the total reward gathered by the online algorithm is
\begin{equation} \label{Feq1}
\sum_i f(\hat{y}(i)) \geq \dfrac{L - 4(\Delta+n-1)}{2(\Delta+n)}~ \sum_{i} x_{T_i}~. 
\end{equation}

Let $f(y^*(i))$ be the reward collected by the optimal offline algorithm in phase $i$. Since the online algorithm optimizes for the total reward, one possibility it considers is to forego reward in the first $(\Delta+n-1)$ steps in each phase in order to traverse the same sequence as the optimal algorithm in the remaining steps. Thus, we must have 
\begin{equation} \label{Feq2} \sum_i f(\hat{y}(i)) \geq \sum_i f(y^*(i)) - (\Delta + n - 1)\sum_{i} x_{T_i}~~.
\end{equation}
Combining \eqref{Feq1} and \eqref{Feq2}, we note for even $L$
$$ \dfrac{\sum_i f(y^*(i))}{\sum_i f(\hat{y}(i))}
\leq 1 + \dfrac{2(\Delta+n)(\Delta+n-1)}{L-4(\Delta+n-1)}~.$$ 
Accounting for $L$ that are not multiples of $4(\Delta+n-1)$, we conclude the competitive ratio of Peek Reset
$$\rho \leq 1 + \dfrac{2(\Delta+n)(\Delta+n-1)}{L-8(\Delta+n-1)+1}~~. 
$$
\end{proof}

\section{Lower Bounds}


\begin{theorem}\label{Theorem5}
The competitive ratio of any deterministic online algorithm on $n^{th}$ order (time-varying) Markov chain models with diameter $\Delta$ for latency $L$ is greater than
$$1 + \dfrac{\tilde{\Delta}}{L} \left(1 + \dfrac{\tilde{\Delta} + L -1}{(\tilde{\Delta}+L-1)^2 + \tilde{\Delta}}\right)~~.$$
In particular, when $n=1$, $\Delta=1$, the ratio is larger than
$$1 + \dfrac{1}{L} + \dfrac{1}{L^2 + 1}~~~.$$
\end{theorem}
\begin{proof}
We motivate the main ideas of the proof for the specific setting of $n=2$ and unit diameter. The extension to general $n$ and unit diameter is then straightforward. Finally, we conjure an example to prove the lower bound for arbitrary $n$ and $\Delta$ via a prismatic polytope construction.  

First consider the case $n=2$ and $\Delta = 1$. We design a $3 \times (L+3)$ matrix initialized as shown below: each row corresponds to a different state, each column corresponds to a time, $``?"$ indicates that the corresponding entry is not known since it lies outside the current peek window of length $L+1$, and $a > 0$ is a variable that will be optimized later. 
\begin{equation} \label{Eeq1}
\begin{bmatrix} 
\Box0 & 1 & a & a & \ldots & a & ? & ? \\
~~~0 & 1 & a & a & \ldots & a & ? & ?\\
~~~0 & 1 &  a & a & \sunderb{5.5em}{(L-1) \text{ terms}} \ldots & a & ? & ? 
\end{bmatrix}
\end{equation}

The box in front of the first entry indicates that the online algorithm made a transition to state 1 from a dummy start state $``*"$ and is ready to make a decision in the current step $t=0$ about whether to continue staying in state 1, or transition to either state 2 or 3. At time $t=0$, the rewards for the next $L+1$ steps are identical, so without loss of generality, let the online algorithm choose the first state, get 0 as reward, and move to the next time $t=1$.  An additional column is revealed and we get the following snapshot.

\begin{equation} \label{Eeq2}
\begin{bmatrix} 
\Box0 & \Box1 & a & a & \ldots & a & 0 & ? \\
~~~0 & ~~~1 & a & a & \ldots & a & 2a & ?\\
~~~0 & ~~~1 &  a & a & \ldots & a & 2a & ?
\end{bmatrix}
\end{equation}

Since $n=2$, we may enforce the following second order Markov dependencies for $t \geq 1$: any state $s \in \{1, 2, 3\}$ yields zero reward unless the previous two states $s', s'' \in \{1, 2, 3, *\}$ were such that $s' \in \{*, s\}$ and $s'' = s$. If this condition is true, then the algorithm receives the current entry pertaining to $s$ as the reward. In other words, other than the special case of dummy start state being one of the states, the algorithm receives the reward only if $s$ is same as the previous two states.

Suppose the online algorithm selects state 1 again at $t=1$. Then it collects reward 1, and another column is revealed as shown below. 

\begin{equation} \label{Eeq3}
\begin{bmatrix} 
\Box0 & \Box1 & \Box a & a & \ldots & a & 0 & 0 \\
~~~0 & ~~~1 & ~~~a & a & \ldots & a & 2a & 0\\
~~~0 & ~~~1 &  ~~~a & a & \ldots & a & 2a & 0
\end{bmatrix}
\end{equation}

In this scenario, the maximum reward the online algorithm can fetch, during its entire execution, is at most $1 + (L-1)a$. To see this note that this is exactly the reward the algorithm gets if it sticks to state 1 at all subsequent times $t$. If, however, it were to jump to any other state and continue with it for at least one step, then it would lose rewards in successive steps due to second order dependency, for a total loss of reward $2a$. All other possibilities incur a loss greater than $2a$. This loss offsets the additional $2a$ reward available with states other than 1. On the other hand, the offline algorithm would select a state $s \in \{2, 3\}$ from the very beginning, and thus receive $1 + (L+1)a$ in total. The competitive ratio in this scenario, therefore, turns out to be  
$$r_1 = \dfrac{1 + (L+1)a}{1 + (L-1)a} = 1 + \dfrac{2a}{1 + (L-1)a}~~.$$

Suppose instead the online algorithm transitions to some state $s \in \{2, 3\}$ at $t=1$. We assume without loss of generality that the algorithm transitions to state 2. The last column is then revealed as follows.\footnote{Note that since our objective here is to prove a lower bound, we would like the competitive ratio to be as high as possible. It might be tempting to set a reward larger than $a$ for state 3 in the last column. That would imply both the online and the offline algorithms could receive an additional reward worth $a$. This, however, would not improve the competitive ratio for the simple reason that for positive x, y, and c,  $\dfrac{x+c}{y+c} > \dfrac{x}{y}$ only if $x < y$ (we instead have $x > y$ since $r_2 >1$).}   

\begin{equation} \label{Eeq4}
\begin{bmatrix} 
\Box0 & \Box1 & ~~~a & a & \ldots & a & 0 & 0 \\
~~~0 & ~~~1 & \Box a & a & \ldots & a & 2a & 0\\
~~~0 & ~~~1 &  ~~~a & a & \ldots & a & 2a & a\\
\end{bmatrix}
\end{equation}

Note that the online algorithm loses on rewards 1 and $a$ in successive steps due to transition. The maximum total reward possible in this case is $La$ regardless of whether the online algorithm makes a transition to other states, or sticks with state 2 subsequently. The offline algorithm, in contrast, would receive all rewards available in state 3. Thus, the ratio in this scenario is 
$$r_2 = \dfrac{1 + (L+2)a}{La} = 1 + \dfrac{1+2a}{La}~~.$$ 
Combining the two cases, the competitive ratio of the online algorithm is at least $\min\{r_1, r_2\}$, and thus we could set $r_1 = r_2$ and solve for $a$.

We can extend this analysis to the general $n \geq 1$ setting with unit diameter easily. We design a $3 \times (L+3)$ matrix with the same row initialization as in \eqref{Eeq1}. Also, we assume that prior to time $t=0$, only zero reward transitions were available between some dummy states\footnote{Another way to enforce the same effect, without the dummy states, is to add additional $n$ columns, with all zero rewards for all the actual states, prior to time $t=0$.}  for both the online and the offline algorithms. We denote the set of these dummy states by ${**}$. We enforce the following $n^{th}$ order Markov dependencies for $t \geq 1$: any state $s \in [m]$ yields zero reward unless the previous $n$ states were same as $s$ or had a prefix consisting only of states in $**$ followed by $s$ in the remaining time steps. If this condition is satisfied, the algorithm receives the current entry pertaining to $s$ as reward.  

The evolution of the reward matrix is as follows. Assuming state 1 was selected at $t=0$, we let column $L+2$ have all entries in rows 2 and 3 to $na$ (instead of $2a$ that we set in \eqref{Eeq2}) at $t=1$. Finally, if the online algorithm selects state 1 at $t=1$, we set the last column to all zeros at time $t=2$ as in \eqref{Eeq3}; otherwise, we set first two entries in the last column to 0, and $a$ in the last row as in \eqref{Eeq4}.

Reasoning along the same lines as before, the competitive ratio of the online algorithm is at least 
\begin{equation} \label{Eeq5} \min\bigg\{1 + \dfrac{na}{1 + (L-1)a}~, 1 + \dfrac{1+na}{La}\bigg\}~~. \end{equation}
We set
\begin{eqnarray*}
1 + \dfrac{na}{1 + (L-1)a} & =  1 + \dfrac{1+na}{La}~,
\end{eqnarray*}
whereby
\begin{eqnarray*}
a  =  \dfrac{n+L-1 + \sqrt{(n+L-1)^2 + 4n}}{2n}~~.
\end{eqnarray*}
Substituting this value for a in \eqref{Eeq5}, and leveraging that 
$$a < \dfrac{n+L-1}{n} + \dfrac{1}{n+L-1}~~,$$
we note the competitive ratio is at least
\begin{equation} \label{lowera} 1 + \dfrac{n}{L} \left(1 + \dfrac{n + L -1}{(n+L-1)^2 + n}\right)~~. \end{equation}

The foregoing analysis may be visualized geometrically in terms of a triangle, with each vertex corresponding to a state. The rewards for initial $L+1$ steps are all same, and thus the online algorithm does not have preference for any state initially. Without loss of generality, as soon as it selects state 1 (with all rewards at time $t=0$ being 0), the rewards for time step $L+2$ are chosen at $t=1$ such that states 2 and 3 would fetch reward $na$ while state 1 will fetch none. The online algorithm could either stay with state 1 and get a suboptimal total reward or jump to an adjacent vertex or state, which would not yield reward for $n$ steps.       

We now extend this analysis to accommodate any finite $\Delta \geq 1$. Toward that goal,  we consider a $\Delta$-dimensional prismatic polytope\footnote{Note that a $d$-dimensional prismatic polytope is constructed from two ($d$ - 1)-dimensional polytopes, translated into the next dimension.} 
with a triangular base (i.e. having $3$ vertices). Each vertex of the polytope corresponds to a state, and the maximum distance between two vertices is exactly $\Delta$. Moreover, for every vertex there is some vertex at distance $d$ for each $d \in [\Delta]$.  The polytope is completely symmetric with respect to all the vertices, and we again set rewards for the first $L+1$ steps at all vertices to be the same as before.

Without loss of generality, we again assume that the online algorithm starts at some state 1 (arbitrary labeled). At the next time step, the reward at all vertices that are at a distance $d$ from this vertex is set to $(n+d-1)a$. Thus, the vertices adjacent to state 1 have reward $na$ in column $(L+2)$ since they lie at distance $d=1$, while the reward for state 1 in this column is 0. Thus, the maximum reward is available at distance $d=\Delta$ from state 1, however, the online algorithm will need to make $\Delta$ steps to reach such a state, and then wait another $n-1$ steps before availing this reward. Thus, effectively, $\tilde{\Delta} = n+\Delta-1$ steps are wasted that the offline algorithm could fully exploit due to prescience. Proceeding along same lines as before, and replacing $n$ with $\tilde{\Delta}$ in \eqref{lowera}, we conclude that the competitive ratio of any deterministic online algorithm on our construction  is at least         
$$1 + \dfrac{\tilde{\Delta}}{L} \left(1 + \dfrac{\tilde{\Delta} + L -1}{(\tilde{\Delta}+L-1)^2 + \tilde{\Delta}}\right)~~.$$

\end{proof}

\begin{theorem} \label{RandomizedLower}
For any $\epsilon > 0$, the competitive ratio of any randomized online algorithm, that is allowed latency $L$, on $n^{th}$ order (time-varying) Markov chain models with $\Delta = 1$ is at least $$1 + \dfrac{(1-\epsilon)n}{L+\epsilon n}~.$$

For a general diameter $\Delta$, the competitive ratio is at least 
$$ 1 + \dfrac{\left(2^{\Delta-1} \lceil 1/\epsilon \rceil - 1 \right)n}{2^{\Delta-1} \lceil 1/\epsilon \rceil L + n}~.$$
\end{theorem}
\begin{proof}
First consider the unit diameter setting (i.e. $\Delta = 1$). We design a matrix with $\lceil 1/{\epsilon} \rceil$ rows and $L+2$ columns. The first column consists of all zeros, the next $L$ columns contain all ones, and the last column contains all zeros except one randomly chosen row that contains $n$. We again enforce the Markov dependency structure described in the proof of Theorem \ref{Theorem5} for all states (or rows) in $[\lceil 1/{\epsilon} \rceil]$.  

The optimal offline algorithm knows beforehand which row $q$ contains $n$ in the last column, and thus collects a total reward $L+n$. On the other hand, any randomized online algorithm chooses this row at $t=0$ with only probability $\epsilon$. Selecting any other row at $t=0$ may fetch a maximum reward of $L$ accounting for all the possibilities including sticking to this row subsequently, or moving to $q$ in one or more transitions. Since the randomized algorithm is assigned at time $t=0$ with the remaining probability $(1-\epsilon)$ to some row other than $q$, its expected reward cannot exceed
$$\epsilon * (L+n) + (1-\epsilon) * L = L + \epsilon n.~$$
Thus, when $\Delta=1$, the competitive ratio for any randomized online algorithm is at least $\dfrac{L+n}{L+\epsilon n}$~ as claimed. For the general setting, we consider a $\Delta$-dimensional prismatic polytope 
with  the base containing $\lceil 1/{\epsilon} \rceil$ vertices. In addition to the usual prismatic polytope topology (assuming bidirectional edges between any pair of adjacent vertices), we add edges so that vertices on each face are strongly connected, i.e., directed edges in both directions connect all pairs of vertices that lie on a face. The polytope contains $u = 2^{\Delta-1} \lceil 1/{\epsilon} \rceil$ states in total. We design a matrix having these many rows and $L+2$ columns as before. 
Any randomized online algorithm has only a $1/u$ probability of getting the maximum possible $L+n$ reward (due to selecting $q$ and sticking with it), and must forfeit a reward no less than $n$ with the remaining probability. Thus, the expected reward cannot exceed $$(L+n)/u + (1-1/u)*L = L + n/u,$$
while the maximum possible reward is $L+n$. Thus the competitive ratio is at least $\dfrac{L+n}{L + n/u}$ which simplifies to the result stated in the problem statement. 
\end{proof}



\end{document}